\documentclass[twoside]{article}

\PassOptionsToPackage{hyphens}{url}\usepackage{hyperref}
\usepackage{comment}
\usepackage{setspace}
\usepackage{amsmath, amssymb, amsthm}
\usepackage{graphicx}
\usepackage{algorithm, algcompatible}
\usepackage[letterpaper,top=2cm,bottom=2cm,left=2.5cm,right=2.5cm,marginparwidth=1.75cm]{geometry}
\newtheorem{prop}{Proposition}

\newcommand{\bp}{\mathbf{p}}
\newcommand{\bq}{\mathbf{q}}

\newcommand{\bv}{\mathbf{v}}
\newcommand{\bx}{\mathbf{x}}
\newcommand{\bz}{\mathbf{z}}
\newcommand{\bM}{\mathbf{M}}

\newcommand{\bV}{\mathbf{V}}
\newcommand{\bX}{\mathbf{X}}

\newcommand{\CA}{\mathcal{A}}

\newcommand{\CH}{\mathcal{H}}
\newcommand{\CO}{\mathcal{O}}

\providecommand{\norm}[1]{\left\lVert#1\right\rVert}
\providecommand{\indic}[1]{\mathbf{1}\left\{#1\right\}}

\newcommand{\bbR}{\mathbb{R}}

\DeclareMathOperator{\E}{E}
\DeclareMathOperator{\Var}{Var}

\newcommand{\hn}{\hat{n}}
\newcommand{\hp}{\hat{p}}

\newcommand{\hC}{\hat{C}}

\newcommand{\tp}{\tilde{p}}
\newcommand{\tbp}{\tilde{\bp}}

\newcommand{\hbv}{\hat{\bv}}

\newcommand{\bac}{\bar{c}}

\newcommand{\bap}{\bar{p}}

\newcommand{\baC}{\bar{C}}

\newcommand{\hbp}{\hat{\bp}}

\newcommand{\eps}{\varepsilon}

\newtheorem{thm}{Theorem}
\newtheorem{cor}{Corollary}
\newtheorem{lem}{Lemma}

\makeatletter
\newenvironment{mechanism}[1][htb]{%
    \renewcommand{\ALG@name}{Mechanism}
   \begin{algorithm*}[#1]%
  }{\end{algorithm*}}
\makeatother

\usepackage[round]{natbib}

\begin{document}

%

%

\title{Near-optimal algorithms for private estimation and sequential testing of collision probability}

\author{%
   R\'obert \mbox{Busa-Fekete}\\
   Google Research, NY, USA \\
   \texttt{busarobi@google.com} \\
  \and
   Umar Syed \\
   Google Research, NY, USA \\
   \texttt{usyed@google.com} \\
}
\date{}
\maketitle

\begin{abstract}
  We present new algorithms for estimating and testing \emph{collision probability}, a fundamental measure of the spread of a discrete distribution that is widely used in many scientific fields. We describe an algorithm that satisfies $(\alpha, \beta)$-local differential privacy and estimates collision probability with error at most $\eps$ using $\tilde{O}\left(\frac{\log(1/\beta)}{\alpha^2 \eps^2}\right)$ samples for $\alpha \le 1$, which improves over previous work by a factor of $\frac{1}{\alpha^2}$. We also present a sequential testing algorithm for collision probability, which can distinguish between collision probability values that are separated by $\eps$ using $\tilde{O}(\frac{1}{\eps^2})$ samples, even when $\eps$ is unknown. Our algorithms have nearly the optimal sample complexity, and in experiments we show that they require significantly fewer samples than previous methods.
\end{abstract}

\section{INTRODUCTION}
\label{sec:intro}

A key property of a discrete distribution is how widely its probability mass is dispersed over its support. One of the most common measures of this dispersal is \emph{collision probability}. Let $\bp = (p_1, \ldots, p_k)$ be a discrete distribution. The collision probability of $\bp$ is defined
$
C(\bp) = \sum_{i=1}^k p^2_i.
$

Collision probability takes its name from the following observation. If $X$ and $X'$ are independent random variables with distribution $\bp$ then $C(\bp) = \Pr[X = X']$, the probability that the values of $X$ and $X'$ coincide. If a distribution is highly concentrated then its collision probability will be close to $1$, while the collision probability of the uniform distribution is $1/k$.

Collision probability has played an important role in many scientific fields, although each time it is rediscovered it is typically given a different name. In ecology, collision probability is called the \emph{Simpson index} and serves as a metric for species diversity \citep{simpson1949measurement,n1}. In economics, collision probability is known as the \emph{Herfindahl–Hirschman index}, which quantifies market competition among firms \citep{herfindahl1997concentration}, and also the \emph{Gini diversity index}, a measure of income and wealth inequality \citep{gini1912variabilita}. Collision probability is also known as the \emph{second frequency moment}, and is used in database optimization engines to estimate self join size \citep{cormode2016join}. In statistical mechanics, collision probability is equivalent to \emph{Tsallis entropy of second order}, which is closely related to Boltzmann–Gibbs entropy \citep{tsallis1988possible}. The negative logarithm of collision probability is \emph{R\'enyi entropy of second order}, which has many applications, including assessing the quality of random number generators \citep{TRNG} and determining the number of reads needed to reconstruct a DNA sequence \citep{motahari2013information}. Collision probability has also been used by political scientists to determine the effective size of political parties \citep{LaaksoTaagepera}. 

Collision probability is \emph{not} equivalent to Shannon entropy, the central concept in information theory and another common measure of the spread of a distribution. However, collision probability has a much more intuitive interpretation, and is also easier to estimate. Specifically, estimating the Shannon entropy of a distribution with support size $k$ requires $\Omega\left(\frac{k}{\log k}\right)$ samples \citep{ValiantValiant}, while the sample complexity of estimating collision probability is independent of $k$. Additionally, the negative logarithm of the collision probability of a distribution is a lower bound on its Shannon entropy, and this lower bound becomes an equality for the uniform distribution.

\subsection{Our contributions}

We present novel algorithms for estimating and testing the collision probability of a distribution. 

{\bf Private estimation:} We give an algorithm for estimating collision probability that satisfies \emph{$(\alpha, \beta)$-local differential privacy}.\footnote{Instead of denoting the privacy parameters by $\eps$ and $\delta$, as is common in the privacy literature, we will use them to denote error and probability, as is common in the statistics literature.} As in previous work, our algorithm is \emph{non-interactive}, which means that there is only a single round of communication between users and a central server, and \emph{communication-efficient}, in the sense that each user sends $O(1)$ bits to the server (in fact, just 1 bit). If $\alpha \le 1$ then our algorithm needs $\tilde{O}\left(\frac{\log(1/\beta)}{\alpha^2 \eps^2}\right)$ samples to output an estimate that has $\eps$ additive error, which nearly matches the optimal sample complexity and improves on previous work by an $O\left(\frac{1}{\alpha^2}\right)$ factor \citep{BravoHermsdorff2022}. We also present experiments showing that our mechanism has significantly lower sample complexity in practice.

{\bf Sequential testing:} We give an algorithm for determining whether collision probability is equal to a given value $c_0$ or differs from $c_0$ by at least $\eps > 0$, assuming that one of those conditions holds. Our algorithm needs $\tilde{O}(\frac{1}{\eps^2})$ samples to make a correct determination, which nearly matches the optimal sample complexity. Importantly, $\eps$ is \emph{not} known to the algorithm. In other words, the algorithm automatically adapts to easy cases by drawing fewer samples. In experiments, we show that our algorithm uses significantly fewer samples than a worst-case optimal sequential testing algorithm due to \citep{das_2017}, and also outperforms algorithms designed for the batch setting, in which the number of samples is specified in advance \citep{Cannone}.

There is a common thread connecting our algorithms that also represents the source of their advantage over previous methods. Any set of $n$ samples from a distribution contains $\binom{n}{2} = \Theta(n^2)$ pairs of samples whose values could potentially collide. Our algorithms examine $\Theta(n^2)$ of these pairs when estimating the collision probability of the distribution, while previous methods used only $O(n)$ pairs. Examining more pairs allows our algorithms to extract more information from a given set of samples, but also significantly complicates the algorithms' analysis, since the pairs are not all disjoint and therefore are not statistically independent.

{\bf Private sequential testing:} The hashing technique we use in our private estimation algorithm can be plugged into our sequential testing algorithm, which results in a private sequential tester. This algorithm has a worse dependency on the privacy parameters $\alpha$ and $\beta$ than our private estimator, but depends on the error parameter $\epsilon$ in the same way as our non-private sequential testing algorithm. To our best knowledge, this is the first private sequential testing algorithm for testing collision probability.

For simplicity, in the main body of this paper we state all theorems using big-$O$ notation, reserving more detailed theorem statements and proofs for the Appendix. 

\section{RELATED WORK}

The collision probability of a distribution is equal to its second frequency moment, and frequency moment estimation has been widely studied in the literature on data streams, beginning with the seminal work of \citet{ALON1999137}. The optimal algorithm for estimating the collision probability of a distribution (rather than a non-random data stream) was given by \citet{crouch2016stochastic}, but their algorithm is not private. Locally differentially private estimation of the frequency moments of a distribution was first studied by \citet{butucea2021locally}, who gave a non-interactive mechanism for estimating any positive frequency moment. The sample complexity of their mechanism depends on the support size of the distribution, and they asked whether this dependence could be removed. Their conjecture was affirmatively resolved for collision probability by \citet{BravoHermsdorff2022}, but removing the dependence on support size led to a much worse dependence on the privacy parameter. It has remained an open question until now whether this trade-off is necessary.

Property and closeness testing has a rich literature~\citep{pmlr-v89-acharya19b,pmlr-v31-acharya13a,DiakonikolasKN15,GoRo00,canonne2022topics}, but the sequential setting is studied much less intensively. Existing algorithms for sequential testing almost always define closeness in terms of total variation distance, which leads to sample complexities on the order $O( \sqrt{k} / \epsilon^2)$, where $k$ is the support size of the distribution and the distribution is separated from the null hypothesis by $\epsilon$ in terms of total variation distance \citep{das_2017,AaFlGa21}. By contrast, all of our results are entirely independent of $k$, making our approach more suitable when the support size is very large.


There are several batch testing approaches which are based on collision statistics. Most notably, the optimal uniform testing algorithm of \citet{Paninski03} distinguishes the uniform distribution from a distribution that is $\epsilon$ far from uniform in terms of total variation distance with a sample complexity $\Theta( \sqrt{k} / \epsilon^2)$. However, in the batch setting, the parameter $\epsilon$ is given to the testing algorithm as input.




\section{PRELIMINARIES}
\label{sec:preliminaries}
We study two problems related to learning the collision probability $C(\bp) = \sum_i p^2_i$ of an unknown distribution $\bp = (p_1, \ldots, p_k)$.

In the {\bf private estimation problem}, a set of $N$ users each possess a single sample drawn independently from distribution $\bp$. We are given an error bound $\eps > 0$ and confidence level $\delta \in [0, 1]$. A central server must compute an estimate $\hC$ that satisfies $|\hC - C(\bp)| \le \eps$ with probability at least $1 - \delta$ while preserving the privacy of the users' samples. A \emph{mechanism} is a distributed protocol between the server and the users that privately computes this estimate. The execution of a mechanism can depend on the samples, and the output of a mechanism is the entire communication transcript between the server and the users. Mechanism $M$ satisfies \emph{$(\alpha, \beta)$-local differential privacy} if for each user $i$ and all possible samples $x_1, \ldots, x_N, x'_i$ we have
\begin{align*}
\Pr [M(x_1, \ldots, x_N) \in \CO]
\le e^\alpha \Pr[M(x_1, \ldots, x_{i-1}, x'_i, x_{i+1}, \ldots, x_N) \in \CO] + \beta,
\end{align*}
where $\CO$ is any set of possible transcripts between the server and the users. In other words, if the privacy parameters $\alpha$ and $\beta$ are small then changing the sample of a single user does not significantly alter the distribution of the mechanism's output. Local differential privacy is the strongest version of differential privacy, and is suitable for a setting where the server is untrusted \citep{dwork2014algorithmic}. The \emph{sample complexity} of the mechanism is the expected number of users $n = \E[N]$ whose samples are observed by the server.

In the {\bf sequential testing problem}, we are given a confidence level $\delta \in [0, 1]$ and the promise that exactly one of the following two hypotheses hold: The \emph{null hypothesis} is that $C(\bp) = c_0$, while the \emph{alternative hypothesis} is that $|C(\bp) - c_0| \ge \eps > 0$. An algorithm must decide which hypothesis is correct based on samples from $\bp$. Instead of fixing the number of samples in advance, the algorithm draws independent samples from $\bp$ one at a time, and after observing each sample, decides to either reject the null hypothesis or to continue sampling. If the null hypothesis is false then the algorithm must reject it, and if the null hypothesis is true then the algorithm must not stop sampling, and each of these events must occur with probability at least $1 - \delta$. Importantly, while $c_0$ is known to the algorithm, $\eps$ is not known, and thus the algorithm must adapt to the difficulty of the problem. The \emph{sample complexity} of the algorithm is the number of observed samples $N$ if the null hypothesis is false, a random variable.



\section{PRIVATE ESTIMATION}
\label{sec:priv_est}

In this section we describe a distributed protocol for privately estimating the collision probability of a distribution. In our protocol, a set of users each draw a sample from the distribution, and then share limited information about their samples with a central server, who computes an estimate of the collision probability while preserving the privacy of each user's sample.

As discussed in Section \ref{sec:intro}, the collision probability of a distribution is the probability that two independent samples from the distribution will coincide. Therefore the most straightforward strategy the server could employ would be to collect all the users' samples and count the number of pairs of samples containing a collision. However, this approach would not be privacy-preserving.

Instead, in Mechanism \ref{alg:second} below, each user applies a one-bit hash function to their private sample and shares only their hash value with the server. The server computes a statistic that essentially counts the number of collisions among all pairs of hash values, and then applies a bias correction to form an estimate of the collision probability. To increase the robustness of this estimate, the server first partitions the hash values into groups and uses the median estimate from among the groups. 


The hashing procedure in Mechanism \ref{alg:second} is carefully designed to both preserve user privacy and also yield an accurate estimate. On the one hand, if each user privately chose an independent hash function, then their hash values would be entirely uncorrelated and contain no useful information about the underlying distribution. On the other hand, if every user applied the same hash function to their sample, then the server could invert this function and potentially learn some user's sample. Instead, in Mechanism \ref{alg:second}, the server sends the same hash function to all users, but each user prepends their sample with a independently chosen \emph{salt}, or random integer, before applying the hash function. Salts are commonly used in cryptographic protocols to enhance security, and they play a similar role in our mechanism. The number of possible salts serves as a trade-off parameter between the privacy and accuracy of our mechanism, with more salts implying a stronger privacy guarantee.





\begin{mechanism}[!ht]
\caption{Private estimation for collision probability \label{alg:second}}
\begin{algorithmic}[1] 
\Statex {\bf Parameters:} Expected number of users $n$, desired relative error $\eps_{\text{rel}} \in (0, 1]$, failure probability $\delta \in (0, 1]$, privacy parameters $\alpha \ge 0, \beta \in (0, 1]$.
\State Server draws $N_j$ from $\textrm{Poisson}(m)$ for each $j \in \{1, \ldots, g\}$, where $g = \frac{160 \log \frac1\delta}{\eps^2_{\text{rel}}}$ and $m = \frac{n}{g}$.
\State Server partitions users into $g$ groups, where each group $j$ has size $N_j$.
\State Server transmits to each user $i$ their assigned group $j_i$.
\State Server transmits random hash function $h : \{0, 1\}^* \mapsto \{-1, +1\}$ to each user.
\State Each user $i$ chooses salt $s_i$ uniformly at random from $\{1, \ldots, r\}$, where $r = 6 \left(\frac{e^\alpha + 1}{e^\alpha - 1}\right)^2\log \frac{4}{\beta}$.
\State Each user $i$ draws sample $x_i$ from distribution $\bp$.
\State Each user $i$ sends hash value $v_i = h(\langle j_i, s_i, x_i \rangle)$ to the server.
\State Server lets $V_j = \sum_{i \in I_j} v_i$
for each group $j$, where $I_j$ is the set of users in group $j$.
\State Server lets $C_j = \frac{r(V^2_j - m)}{m^2}$ for each group $j$.
\State Server partitions the $g$ groups into $a = 8 \log \frac1\delta$ supergroups, each containing $b = \frac{g}{a}$ groups.
\State Server lets $\baC_\ell = \frac{1}{b} \sum_{j \in J_\ell} C_j$ for each supergroup $\ell$, where $J_\ell$ is the set of groups in supergroup $\ell$.
\State Server outputs $\hC$, the median of $\baC_1, \ldots, \baC_a$.
\end{algorithmic}
\end{mechanism}

The theorems in this section provide guarantees about the privacy and accuracy of Mechanism \ref{alg:second}. 




\begin{thm} \label{thm:second_privacy} Mechanism \ref{alg:second} satisfies $(\alpha, \beta)$-local differential privacy.
\end{thm}


\begin{thm} \label{thm:second_error} The sample complexity of Mechanism \ref{alg:second} is $n$. Moreover, if
$
n \ge \Omega\left(\left(\frac{e^\alpha + 1}{e^\alpha - 1}\right)^2\frac{\log \frac1\beta \log \frac1\delta}{C(\bp)\eps^2_{\text{rel}} }\right)
$
then the estimate $\hC$ output by Mechanism \ref{alg:second} satisfies $|\hC - C(\bp)| \le \eps_{\textrm{rel}} C(\bp)$ with probability $1 - \delta$.
\end{thm}

To simplify comparison to previous work, we state a straightforward corollary of Theorem \ref{thm:second_error} that converts its relative error guarantee to an absolute error guarantee.  

\begin{cor} \label{cor:second_error} If $\alpha \le 1$, $\eps_{\text{rel}} = \frac{\eps}{C(\bp)} \in (0, 1]$ and
$
n \ge \Omega\left(\frac{C(\bp)\log \frac1\beta \log \frac1\delta}{\alpha^2\eps^2}\right)
$ then the estimate $\hC$ output by Mechanism \ref{alg:second} satisfies $|\hC - C(\bp)| \le \eps$ with probability $1 - \delta$.\end{cor}

\subsection{Lower bound}

The next theorem proves that the sample complexity bound in Corollary \ref{cor:second_error} is tight for small $\alpha$ up to logarithmic factors. 
\begin{thm} \label{thm:private_lower} Let $\hC_{\alpha, n}(\bp)$ be a collision probability estimate returned by an $(\alpha, 0)$-locally differentially private mechanism that draws $n$ samples from distribution $\bp$. If $\eps \le 1, \alpha \le \frac{23}{35}$ and $n \le o\left(\frac{1}{\alpha^2\eps^2}\right)$ then there exists a distribution $\bp$ such that 
$
\E\left[|\hC_{\alpha, n}(\bp) - C(\bp)|\right] \ge \eps.
$\end{thm}

\subsection{Reduction to private distribution estimation}
\label{sec:reduction}

A natural alternative to Mechanism \ref{alg:second} would be an indirect approach that privately estimates the distribution itself, and then computes the collision probability of the estimated distribution. We prove a theoretical separation between these two approaches, by showing that this reduction does not preserve optimality. Specifically, we show that even if the mechanism for private distribution estimation has the optimal sample complexity, using it as a subroutine for collision probability estimation may require a number of samples that depends on the support size of the distribution. By contrast, the sample complexity of our method is independent of support size (see Corollary \ref{cor:second_error}).

Let $[k] = \{1, \ldots, k\}$ be the sample space. Let $\Delta_k$ be the set of all distributions on $[k]$. Let $A : [k]^n \rightarrow \Delta_k$ denote an algorithm that inputs $n$ samples, one per user, and outputs an estimated distribution. An $(\alpha, \beta)$-local differentially private algorithm $A^*$ is \emph{$(\alpha, \beta)$-minimax optimal} if
\[
A^* \in \arg \min_{A \in \CA_{\alpha, \beta}} \max_{\bp \in \Delta^k} \E_{x_1, \ldots, x_n \sim \bp^n}[\norm{A(x_1, \ldots, x_n) - \bp}_1]
\]
where $\CA_{\alpha, \beta}$ is the set of all $(\alpha, \beta)$-local differentially private algorithms for estimating distributions. 

\begin{thm} \label{thm:failure} For all $\alpha \in (0, 1)$ there exists an $(\alpha, 0)$-minimax optimal algorithm $A^* : [k]^n \rightarrow \Delta_k$ and a distribution $\bp \in \Delta^k$ such that if each $x_i \in [k]$ is drawn independently from $\bp$ and
$
\E\left[\left|C(A^*(x_1, \ldots, x_n)) - C(\bp)\right|\right] \le \eps$ then $n \ge \Omega\left(\min\left\{\frac{k^2}{\alpha^2}, \frac{k}{\alpha^2\eps}\right\}\right).
$
\end{thm}

\subsection{Comparison to previous work}
\label{sec:private_comparison}

\citet{butucea2021locally} gave a non-interactive $(\alpha, 0)$-locally differentially private mechanism for estimating collision probability with sample complexity $\tilde{O}\left(\frac{(\log k)^2}{ \alpha^2\eps^2}\right)$ and communication complexity $O(k)$. \citet{BravoHermsdorff2022} gave a non-interactive mechanism with the same privacy guarantee, sample complexity $\tilde{O}\left(\frac{1}{\alpha^4\eps^2}\right)$, and communication complexity $O(1)$.\footnote{Note that \citeauthor{BravoHermsdorff2022}'s original NeurIPS paper claimed $\tilde{O}\left(\frac{1}{\alpha^2\eps^2}\right)$ sample complexity, but a more recent version on Arxiv claims $\tilde{O}\left(\frac{1}{\alpha^4\eps^2}\right)$ sample complexity and explains that the original version contained mistakes. See References for a link to the Arxiv version.} Thus the latter mechanism is better suited to distributions with very large support sizes, but is a worse choice when the privacy parameter $\alpha$ is very small. Our mechanism combines the advantages of these mechanisms, at the expense of a slightly weaker privacy guarantee and an additional $O(C(\bp) \log \frac1\beta)$ samples.

Notably, the earlier mechanism due to \citet{BravoHermsdorff2022} is also based on counting collisions among salted hash values. But there are key differences between the mechanisms which lead to our improved sample complexity. In their mechanism, the server assigns salts to the users, each user adds noise to their hash value, and the server counts hash collisions among $\frac{n}{2}$ disjoint user pairs. In our mechanism, the salts are chosen privately, no additional noise is added to the hash values, and the server counts hash collisions among $\Theta(n^2)$ user pairs. Using more available pairs to count collisions is a more efficient use of data (although it significantly complicates the analysis, as the pairs are not all independent), and choosing the salts privately eliminates the need for additional randomness, which improves the accuracy of the estimate.




\section{SEQUENTIAL TESTING}

In this section we describe an algorithm for sequentially testing whether $C(\bp) = c_0$ (the null hypothesis) or $|C(\bp) - c_0| \ge \eps > 0$ (the alternative hypothesis), where $c_0$ is given but $\eps$ is unknown. 
Algorithm \ref{alg:closeness_simple} below draws samples from the distribution $\bp$ one at a time. Whenever the algorithm observes a sample $x_i$ it updates a running estimate of $|C(\bp) - c_0|$ based on the all-pairs collision frequency observed so far. The algorithm compares this estimate to a threshold that shrinks like $\Theta\big(\sqrt{i^{-1}\log \log i}\big)$ and rejects the null hypothesis as soon as the threshold is exceeded. Although our algorithm is simple to describe, its proof of correctness is non-trivial, as it involves showing that a sequence of dependent random variables (the running estimates) become concentrated. Our proof uses a novel decoupling technique to construct martingales based on the running estimates.

\begin{algorithm}[!h]   
\caption{Sequential testing of collision probability \label{alg:closeness_simple}}
\begin{algorithmic}[1] 
\STATE {\bf Given:} Null hypothesis value $c_0$, confidence level $\delta \in [0, 1]$.
\FOR{$i=1, 2, 3, \dots$}
\STATE Draw sample $x_i$ from distribution $\bp$.
\STATE Let $T_i = \sum_{j=1}^{i-1} \indic{x_i = x_j} - 2(i - 1)c_0$.
\IF{$\left\lvert \frac{2}{i(i-1)}\sum_{j=1}^i T_j \right\rvert > 3.2\sqrt{\frac{\log \log i + 0.72 \log (20.8 /\delta)}{i}}$} \label{alg1:reject}
\STATE Reject the null hypothesis. 
\ENDIF
\ENDFOR
\end{algorithmic}
\end{algorithm}

The next theorem provides a guarantee about the accuracy of Algorithm \ref{alg:closeness_simple}.

\begin{thm}\label{thm:general_closeness_thm_simple}
If $C(\bp) = c_0$ then Algorithm \ref{alg:closeness_simple} does not reject the null hypothesis with probability $1 - \delta$. If $|C(\bp) - c_0| \ge \eps$ then Algorithm \ref{alg:closeness_simple} rejects the null hypothesis after observing $N$ samples, where
$
N \in O\left(\frac{1}{\eps^2} \log \log \frac1\eps \log \frac{1}{\delta}\right)
$
with probability $1 - \delta$.
\end{thm}

The $\log \log \frac1\eps$ factor in Theorem \ref{thm:general_closeness_thm_simple} results from our application of a confidence interval due to \citet{HoRaMcSe21} that shrinks like $\Theta\big(\sqrt{i^{-1}\log \log i}\big)$. Note that $\log \log \frac1\eps < 4$ if $\eps \ge 10^{-10}$, so this factor is negligible for nearly all problem instances of practical interest.

We remark that our proof technique bears some superficial resemblance to the approach used in recent work by \citet{AaFlGa21}. They make use of the fact that for any random variable $T$ taking values from $\mathbb{N}$ and for all $T \in \mathbb{N}_+$, it holds that $\mathbb{E} \left[ T  \right] \le N + \sum_{t\ge N} \mathbb{P} ( T \ge t )$. Then with a carefully selected $N$ and Chernoff bounds with infinite many applications of union bound implies upper bound on the expected sample complexity. By contrast, we construct a test martingale that is specific to collision probability and apply an anytime or time-uniform concentration bound to the martingale introduced by~\citet{WaRa20}. 



\subsection{Lower bound}

The next theorem proves that sample complexity bound in Theorem \ref{thm:general_closeness_thm_simple} is tight up to log-log factors.

\begin{thm}\label{corr:lower_seq} Let $N$ be the number of samples observed by a sequential testing algorithm for collision probability. For all $\eps, \delta \in [0, 1]$ there exists a distribution $\bp$ and $c_0 \in [0, 1]$ such that $|C(\bp) - c_0| \ge \eps$ and if the algorithm rejects the null hypothesis with probability $1 - \delta$ then
$
\E[N] \ge \Omega\left(\frac{\log(1/\delta)}{\eps^2}\right).
$
\end{thm}

\subsection{Comparison to previous work}
\label{sec:sequential_comparison}

\citet{das_2017} described a general sequential testing algorithm that can be adapted to estimate collision probability and has a worst-case sample complexity comparable to Algorithm \ref{alg:closeness_simple}. However, their approach has two major disadvantages relative to ours which lead to much higher sample complexities in practice, and which we empirically confirm in Section \ref{sec:experiments}. First, their approach is based on a simple ``doubling trick'': They repeatedly invoke a non-sequential testing algorithm on subsequences of samples with successively smaller values of the error tolerance $\epsilon$, and stop when the testing algorithm rejects. This is a wasteful use of samples compared to our approach, as stopping cannot occur within a subsequence, and everything learned from previous subsequences is discarded. Second, applying their approach to collision probability testing requires partitioning the $n$ samples into $\frac{n}{2}$ disjoint pairs, so that the observed collisions are independent of each other. By contrast, our approach uses observed collisions among all $\Theta(n^2)$ pairs of samples to estimate collision probability, which significantly complicates the theoretical analysis, but leads to better empirical performance.

\subsection{Private sequential testing}
\label{sec:private_sequential_testing}

In this section, we apply the private estimation technique to sequential testing. The most obvious approach to private sequential testing is to apply Mechanism \ref{alg:second} repeatedly with an exponentially increasing sample size, and then apply the union bound over the repetitions. We call this algorithm the Doubling Tester (as it is based on the doubling trick~\citep{Auer95,CeLu06}), and it has the following sample complexity.
\begin{thm}\label{thm:private_doubling}
If $C(\bp) = c_0$ then the Doubling Tester is $(\alpha, \beta)$-local differential privacy and does not reject the null hypothesis with prob. $1 - \delta$. If $|C(\bp) - c_0| \ge \eps$ then the Doubling Tester rejects the null hypothesis after observing $N_1$ samples if $\alpha>1$ and $N_2$ if $\alpha\le 1$, where 
$
N_1 \in O \left(\left(\frac{e^\alpha + 1}{e^\alpha - 1}\right)^2\frac{\log^2 \frac1\eps \log \frac1\beta \log \frac1\delta}{\eps^2 }  \right)
$
and
$
N_2 \in  O \left(\frac{\log^2 1/ \eps }{\eps^2} \cdot \frac{\log \frac1\beta \log \frac1\delta}{\alpha^2}\right)
$ 
respectively, with prob. $1 - \delta$.
\end{thm}
Another approach is based on the observation that the hashing of Mechanism \ref{alg:second} only depends on the privacy parameters. So we can apply the same hashing to the input of Algorithm \ref{alg:closeness_simple}. The only difference is that the hashing introduces a bias to the test statistic, which has to be corrected, and the scale of the statistic becomes of order $\log \frac1\beta$. We call this the Private Sequential Tester (PSQ), which is described in Appendix \ref{app:private_seq}, and it has the following sample complexity.

\begin{thm}\label{thm:private_seq_sample_comp}
If $C(\bp) = c_0$ then the PSQ algorithm is $(\alpha, \beta)$-local differential privacy and does not reject the null hypothesis with probability $1 - \delta$. If $|C(\bp) - c_0| \ge \eps$ then the PSQ algorithm rejects the null hypothesis after observing $N_1$ samples if $\alpha>1$ and $N_2$ if $\alpha\le 1$, where 
$
N_1 \in O \left(\left(\frac{e^\alpha + 1}{e^\alpha - 1}\right)^4\frac{\log \log \frac1\eps \log^2 \frac1\beta \log \frac1\delta}{\eps^2 }  \right)
$
and
$
N_2 \in  O \left(\frac{\log \log 1/ \eps }{\eps^2} \cdot \frac{\log^2 \frac1\beta \log \frac1\delta}{\alpha^2}\right)
$ 
respectively, with probability $1 - \delta$.
\end{thm}
Based on Theorem \ref{thm:private_seq_sample_comp}, the Private Sequential Tester has lower sample complexity than the Doubling Tester, if $\log \frac1\eps $ is much larger than $\log \frac1\beta$. This is often the case since, for example, for a close-to-uniform distribution with large domain size $d$, in which case the interesting parameter budget is $\epsilon < 1/d$. 

\section{EXPERIMENTS}
\label{sec:experiments}

We compared our mechanism for private collision probability estimation (Mechanism \ref{alg:second}) to the recently proposed mechanism from \citet{BravoHermsdorff2022}. As discussed in Section \ref{sec:private_comparison}, we expect Mechanism \ref{alg:second} to outperform their mechanism when the support size of the distribution is large and the privacy requirement is strict. We also compared to the indirect method described in Section \ref{sec:reduction}: Privately estimate the distribution itself, and then compute the collision probability of the estimated distribution. In our experiments we use an open-source implementation of a private heavy hitters algorithm due to \citet{cormode2021frequency}.\footnote{\url{https://github.com/Samuel-Maddock/pure-LDP}}

In Figure \ref{fig:samp_private} we use each mechanism to privately estimate the collision probability of two distributions supported on 1000 elements: the uniform distribution ($p_i = 1/k$) and the power law distribution ($p_i \propto 1/i$). Our simulations show that Mechanism \ref{alg:second} has significantly lower error for small values of the privacy parameters $\alpha$ and $\beta$.

\begin{figure}[!h]
    \centering
    \includegraphics[width=0.4\textwidth]{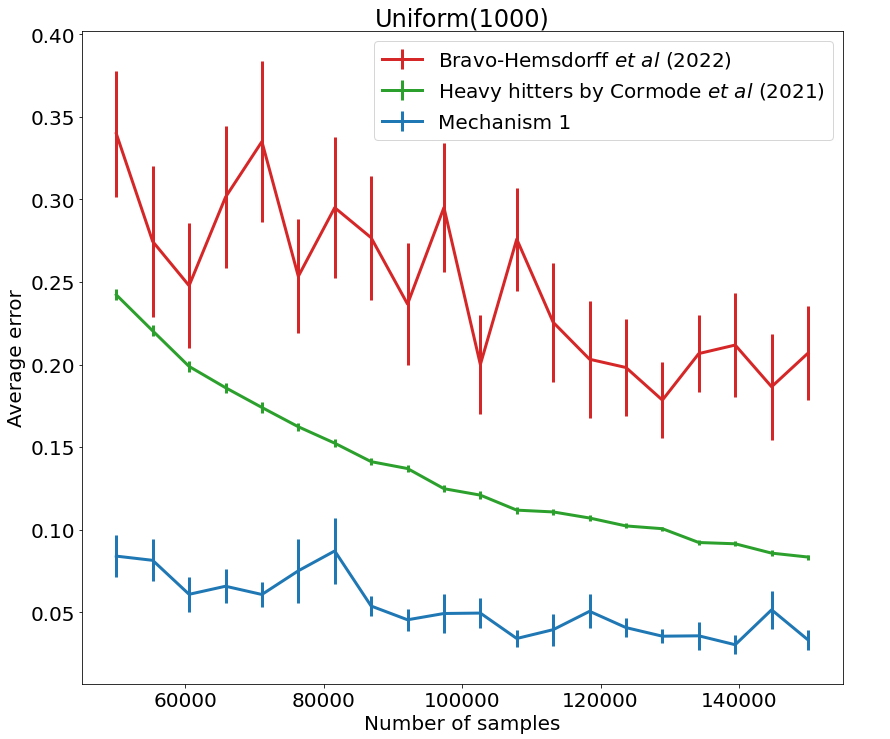}
    \includegraphics[width=0.4\textwidth]{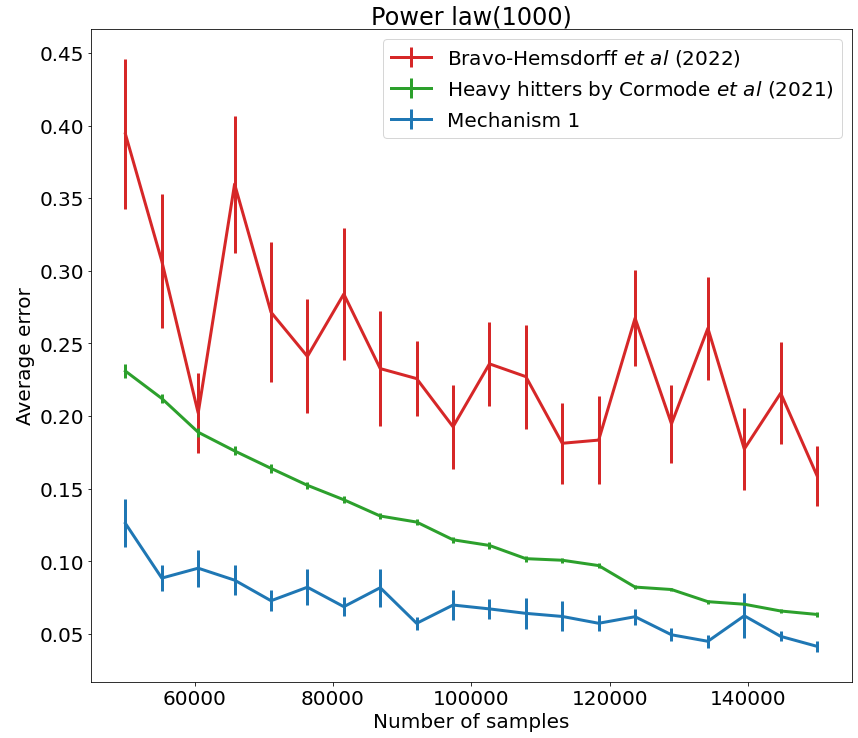}
    \caption{Sample complexity of private collision probability estimation mechanisms for $\alpha = 0.25$. Both mechanisms use the MD5 hash function and confidence level $\delta = 0.1$. For Mechanism 1 we let $\beta = 10^{-5}$. Error bars are one standard error.}
    \label{fig:samp_private}
\end{figure}


We next compared our sequential testing algorithm (Algorithm \ref{alg:closeness_simple}) to the algorithm from \citet{das_2017}. We ran experiments comparing the two algorithms on the power law distribution ($p_i \propto 1/i$) and exponential distribution ($p_i \propto \exp(-i)$), with the results depicted in Figure \ref{fig:samp_das}. Consistent with our discussion in Section \ref{sec:sequential_comparison}, we found that as each tester's null hypothesis approaches the true collision probability, the empirical sample complexity of \citet{das_2017}'s algorithm becomes much larger than the empirical sample complexity of Algorithm \ref{alg:closeness_simple}.

\begin{figure}[!ht]
    \centering
    \includegraphics[width=0.4\textwidth]{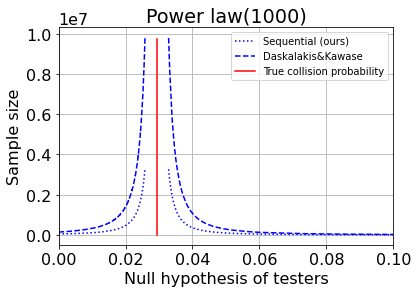}
    \includegraphics[width=0.4\textwidth]{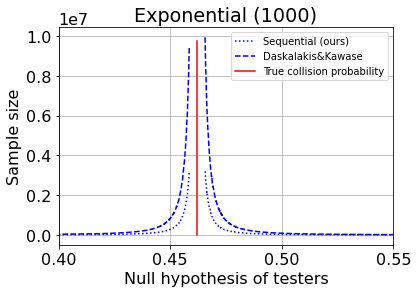}
    \caption{Sample complexity of our sequential tester (Algorithm \ref{alg:closeness_simple}) compared to the sample complexity of \citet{das_2017}'s sequential tester adapted for collision probability testing.}
    \label{fig:samp_das}
\end{figure}

We also compared Algorithm \ref{alg:closeness_simple} to two batch testing algorithms, both of which are described in a survey by \citet{Cannone}:
\begin{itemize}
    \item {\bf Plug-in:} Form empirical distribution $\hat{\bp}$ from samples $x_1, \ldots, x_n$ and let $\hC = C(\hat{\bp})$.
    \item {\bf U-statistics:} Let $\hC = \frac{2}{n(n-1)} \sum_{i < j} \indic{x_i = x_j}$ be the all-pairs collision frequency.
\end{itemize}
Each batch testing algorithm takes as input both the null hypothesis value $c_0$ and a tolerance parameter $\eps$, and compares $|\hC - c_0|$ to $\eps$ to decide whether to reject the null hypothesis $C(\bp) = c_0$. The sample complexity of a batch testing algorithm is determined via worst-case theoretical analysis in terms of $\eps$. On the other hand, sequential testing algorithms automatically adapt their sample complexity to the difference $|C(\bp) - c_0|$.

In Figure \ref{fig:samp_batch} we evaluate batch and our sequential testing algorithms on both on the uniform distribution and power law distributions. We use 20 different support sizes for each distribution, evenly spaced on a log scale between $10$ and $10^6$ inclusively. Varying the support size also varies $|C(\bp) - c_0|$.

As expected, when $|C(\bp) - c_0|$ is large, our sequential testing algorithm requires many fewer samples than the batch algorithm to reject the null hypothesis, and as $|C(\bp) - c_0|$ shrinks the number of samples required sharply increases (see grey areas in Figure \ref{fig:samp_batch}). In all cases our sequential testing algorithm is never outperformed by the batch testing algorithms.

\begin{figure}[!ht]
    \centering
    \includegraphics[width=0.38\textwidth]{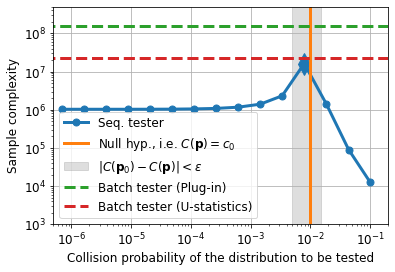}
    \includegraphics[width=0.38\textwidth]{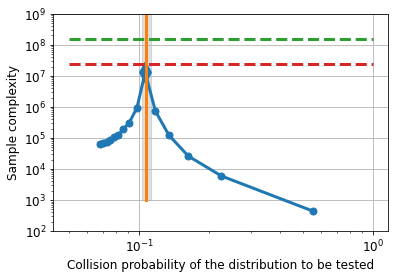}
    \caption{Sample complexity of the sequential tester compared to the sample complexity of the batch testers. For the batch testers, the tolerance parameter $\epsilon$ is set to $0.01$.}
    \label{fig:samp_batch}
\end{figure}


Note that in Figure \ref{fig:samp_batch} the plug-in tester has a worse sample complexity than the U-statistics tester. Since these sample complexities are determined by theoretical analysis, we experimentally confirmed that this discrepancy is not simply an artifact of the analysis. In Figure \ref{fig:batch_comp_est} we run simulations comparing the algorithms in terms of their error $|\hC - C(\bp)|$, and find that the plug-in tester is also empirically worse than the U-statistics tester.

\begin{figure}[!h]
     \centering
     \includegraphics[width=0.45\columnwidth]{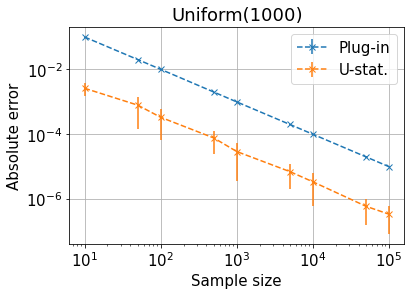}
     \includegraphics[width=0.45\columnwidth]{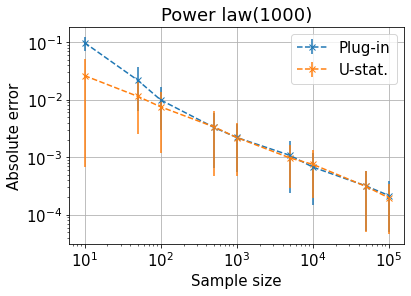}
     \caption{Empirical absolute error of plug-in and U-statistic estimators when the data is generated from uniform distribution and power law with domain size 1000.}
    \label{fig:batch_comp_est}        
\end{figure}

\subsection{Private sequential tester}

\begin{figure}[!ht]
     \centering
         \centering
         \includegraphics[width=0.45\columnwidth]{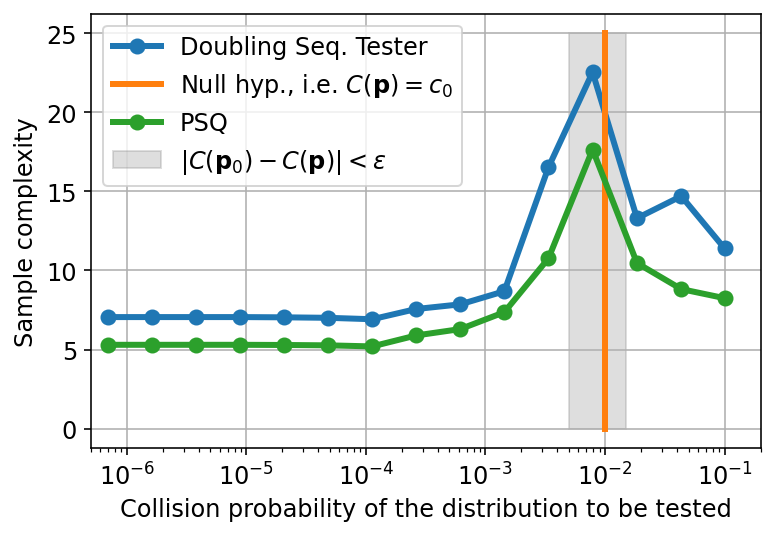}
         \includegraphics[width=0.45\columnwidth]{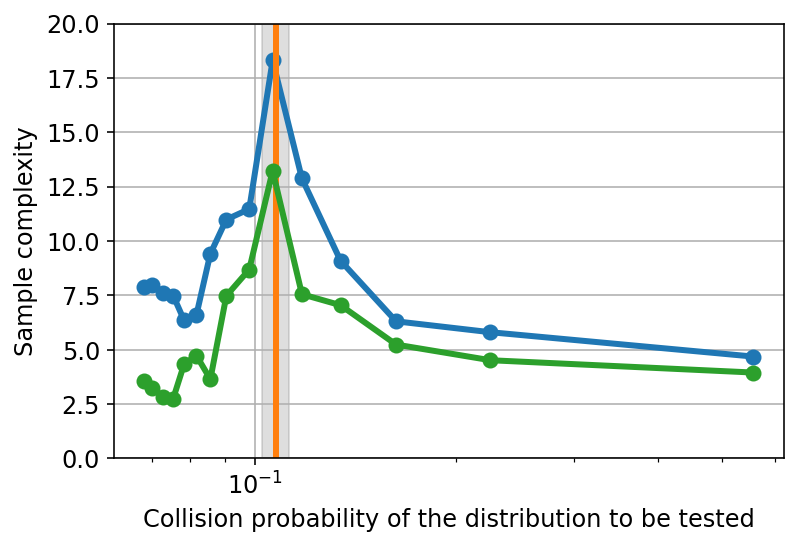}
     \caption{Sample complexity of private sequential testing algorithms with respect to the non-private estimator. The sample complexity of private sequential testers is divided by the sample complexity of the non-private sequential tester from Figure \ref{fig:samp_batch} and the multiplicative factors are shown.}
    \label{fig:priv_samp_comp}        
\end{figure}

In the last set of experiments, we evaluated the private sequential testing algorithm that is introduced in Subsection \ref{sec:private_sequential_testing} and defined in Algorithm \ref{alg:psq}. We refer to this algorithm as PSQ. As a baseline, we ran Mechanism \ref{alg:second} with the so-called ``doubling trick'', with the initial sample complexity set to $2^{13}$, and if the confidence interval of the estimator was not tight enough to make decision, \emph{i.e.} reject or accept, then we doubled the sample complexity and re-ran Mechanism \ref{alg:second}. We refer to this approach as Doubling Sequential Tester. The sample complexity of private sequential algorithms are show in Figure \ref{fig:priv_samp_comp} on the same problem instance that were used for evaluating the non-private testers in Figure \ref{fig:batch_comp_est}. Unsurprisingly, private testers require more samples than the non-private ones, and so we plotted the multiplicative factor between the sample complexity of private and non-private testers. The Doubling Sequential Tester requires 10x more samples if the distribution to be tested close to the null hypothesis. The PSQ algorithm typically requires fewer samples, however it still needs 3-17x times more samples than the non-private tester.

\section{CONCLUSIONS AND FUTURE WORK}

We introduced a locally differentially private estimator for collision probability that is near-optimal in a minimax sense and empirically superior to the state-of-the-art method introduced by \citet{BravoHermsdorff2022}. Our method is based on directly estimating the collision probability using all pairs of observed samples, unlike in previous work. We also introduced a near-optimal sequential testing algorithm that is likewise based on directly estimating the collision probability, and requires far fewer samples than minimax optimal testing algorithms for many problem instances and can be combined with our private estimator.


\bibliographystyle{plainnat}

%

%


\newpage
{\LARGE \bf Supplementary material for ``Near-optimal algorithms for private estimation and sequential testing of collision probability''}

\allowdisplaybreaks

\section{Proof of Theorem \ref{thm:second_privacy}}
\label{app:second_privacy}

Let $X_i$ and $V_i$ be the private value and hash value, respectively, for user $i$, and let $H$ be the random hash function chosen by the server.\footnote{Note that this definition of $V_i$ overrides the definition given in Mechanism \ref{alg:second}. We will not use the latter definition in this proof.} Let $\bX = (X_1, \ldots, X_n)$ and $\bV = (V_1, \ldots, V_n)$.


Recall that in the local model of differential privacy, the output of a mechanism is the entire communication transcript between the users and the server, which in the case of Mechanism \ref{alg:second} consists of the hash function and all the hash values.\footnote{Strictly speaking, the transcript also includes the group assignments. However, these assignments can be chosen arbitrarily, as long as each group has the desired size, and do not impact the privacy of the mechanism in any way. Hence we omit them from the transcript for simplicity.} Therefore our goal is to prove that for any subset $\CO$ of possible values for $(\bV, H)$ we have
\begin{equation}
\Pr[(\bV, H) \in \CO ~|~ \bX = \bx] \le e^\alpha \Pr[(\bV, H) \in \CO ~|~ \bX = \bx'] + \beta. \label{eq:goal}    
\end{equation}
where $\bx$ and $\bx'$ differ in one component.

The proof of Eq.~\eqref{eq:goal} will need a couple of observations. First, given the hash function and private values, the hash values are mutually independent:
\begin{equation}
\Pr[\bV = \bv ~|~ H = h \wedge \bX = \bx] = \prod_i \Pr[V_i = v_i ~|~ H = h \wedge X_i = x_i]. \label{eq:mutual_ind}
\end{equation}
Second, the hash function is independent of the private values:
\begin{equation}
\Pr[H = h ~|~ \bX = \bx] = \Pr[H = h ~|~ \bX = \bx']. \label{eq:ind}
\end{equation}
We also need a definition. Fix user $i^*$. Suppose $\bx$ and $\bx'$ differ in component $i^*$. Define $H_\alpha$ to be the set of all hash functions such that if $h \in H_\alpha$ then
\begin{equation}
\Pr[V _{i^*} = v ~|~ H = h \wedge X _{i^*} = x _{i^*}] \le e^\alpha \Pr[V _{i^*} = v ~|~ H = h \wedge X _{i^*} = x' _{i^*}] \label{eq:upper}
\end{equation}
for all $v$. Note that the definition of $H_\alpha$ depends implicitly on $x _{i^*}$ and $x' _{i^*}$, but does not depend on $X _{i^*}$.

Combining the above we have
\begin{align*}
    & ~\Pr[\bV = \bv \wedge H = h ~|~ \bX = \bx \wedge H \in \CH_\alpha]\\
    =& ~ \Pr[\bV = \bv ~|~ H = h \wedge \bX = \bx \wedge H \in \CH_\alpha] \cdot \Pr[H = h ~|~ \bX = \bx \wedge H \in \CH_\alpha]\\
    =& ~ \prod_i \Pr[V_i = v_i ~|~ H = h \wedge X_i = x_i \wedge H \in \CH_\alpha] \cdot \Pr[H = h ~|~ \bX = \bx \wedge H \in \CH_\alpha] & \because \textrm{Eq.~\eqref{eq:mutual_ind}}\\
    =& ~ \prod_i \Pr[V_i = v_i ~|~ H = h \wedge X_i = x_i \wedge H \in \CH_\alpha] \cdot \Pr[H = h ~|~ \bX = \bx' \wedge H \in \CH_\alpha]  & \because \textrm{Eq.~\eqref{eq:ind}}\\
    \le& ~ e^\alpha \prod_i \Pr[V_i = v_i ~|~ H = h \wedge X_i = x'_i \wedge H \in \CH_\alpha] \cdot \Pr[H = h ~|~ \bX = \bx' \wedge H \in \CH_\alpha] & \because \textrm{Eq.~\eqref{eq:upper}}\\
    =& ~ e^\alpha \Pr[\bV = \bv ~|~ H = h \wedge \bX = \bx' \wedge H \in \CH_\alpha] \cdot \Pr[H = h ~|~ \bX = \bx' \wedge H \in \CH_\alpha] & \because \textrm{Eq.~\eqref{eq:mutual_ind}}\\
    =& ~ e^\alpha \Pr[\bV = \bv \wedge H = h ~|~ \bX = \bx' \wedge H \in \CH_\alpha].
\end{align*}
Summing both sides over $\CO$ yields
\begin{equation}
\Pr[(\bV, H) \in \CO ~|~ \bX = \bx \wedge H \in \CH_\alpha] \le e^\alpha \Pr[(\bV, H) \in \CO ~|~ \bX = \bx' \wedge H \in \CH_\alpha]. \label{eq:alpha}
\end{equation}
Therefore we have
\begin{align*}
    \Pr[(\bV, H) \in \CO ~|~ \bX = \bx] =&~ \Pr[(\bV, H) \in \CO ~|~ \bX = \bx \wedge H \in \CH_\alpha] \cdot \Pr[H \in \CH_\alpha ~|~ \bX = \bx]\\
    &~+ \Pr[(\bV, H) \in \CO ~|~ \bX = \bx \wedge H \not\in \CH_\alpha] \cdot \Pr[H \not\in \CH_\alpha ~|~ \bX = \bx]\\
    \le&~ \Pr[(\bV, H) \in \CO ~|~ \bX = \bx \wedge H \in \CH_\alpha] \cdot \Pr[H \in \CH_\alpha ~|~ \bX = \bx]\\
    &~+ \Pr[H \not\in \CH_\alpha ~|~ \bX = \bx]\\
    =&~ \Pr[(\bV, H) \in \CO ~|~ \bX = \bx \wedge H \in \CH_\alpha] \cdot \Pr[H \in \CH_\alpha ~|~ \bX = \bx'] & \because \textrm{Eq.~\eqref{eq:ind}}\\
    &~+ \Pr[H \not\in \CH_\alpha ~|~ \bX = \bx]\\
    \le&~ e^\alpha \Pr[(\bV, H) \in \CO ~|~ \bX = \bx' \wedge H \in \CH_\alpha] \cdot \Pr[H \in \CH_\alpha ~|~ \bX = \bx']  & \because \textrm{Eq.~\eqref{eq:alpha}}\\
    &~+ \Pr[H \not\in \CH_\alpha ~|~ \bX = \bx]\\
    =&~ e^\alpha \Pr[(\bV, H) \in \CO \wedge H \in \CH_\alpha ~|~ \bX = \bx'] + \Pr[H \not\in \CH_\alpha ~|~ \bX = \bx]\\
    \le&~ e^\alpha \Pr[(\bV, H) \in O ~|~ \bX = \bx'] + \Pr[H \not\in \CH_\alpha ~|~ \bX = \bx]\\
    =&~ e^\alpha \Pr[(\bV, H) \in O ~|~ \bX = \bx'] + \Pr[H \not\in \CH_\alpha] & \because \textrm{Eq.~\eqref{eq:ind}}
\end{align*}
It remains to show that $\Pr[H \not\in \CH_\alpha] \le \beta$. Recall that $x _{i^*} \neq x' _{i^*}$ are the only different values in $\bx$ and $\bx'$. For any hash value $v$ and salt $s$ define the random variables
\begin{align*}
p_{s, v}(H) &= \indic{H(\langle j_{i^*}, s, x _{i^*} \rangle) = v}\\
p'_{s, v}(H) &= \indic{H(\langle j_{i^*}, s, x' _{i^*} \rangle) = v}
\end{align*}
Also define $\bap_v(H) = \frac1r \sum_s p_{s, v}(H)$ and $\bap'_v(H) = \frac1r \sum_s p'_{s, v}(H)$. Observe that
\begin{align*}
\Pr[V _{i^*} = v ~|~ H = h \wedge X _{i^*} = x _{i^*}] = \bap_v(h)\\
\Pr[V _{i^*} = v ~|~ H = h \wedge X _{i^*} = x' _{i^*}] = \bap'_v(h)
\end{align*}
since each user chooses their salt uniformly at random from $\{1, \ldots, r\}$. Therefore
\begin{equation}
\Pr[H \not\in \CH_\alpha] = \Pr\left[\exists v : \frac{\bap_v(H)}{\bap'_v(H)} > e^\alpha\right]. \label{eq:prob}
\end{equation}
We will analyze the right-hand side of Eq.~\eqref{eq:prob}. Clearly $\E[p_{s, v}(H)] = \E[p'_{s, v}(H)] = \frac12$, since each $H(\langle j, s, x \rangle)$ is chosen uniformly at random from $\{-1, +1\}$. Also $\bap_v(H)$ and $\bap'_v(H)$ are each the average of $r$ independent Boolean random variables, since each $H(\langle j, s, x \rangle)$ is chosen independently. Therefore, by the Chernoff bound, for all $\eps \in [0, 1]$ and any hash value $v$
\begin{align*}
\Pr\left[\bap_v(H) \ge \frac12(1 + \eps)\right] &\le \exp\left(-\frac{\eps^2r}{6}\right)\\
\Pr\left[\bap'_v(H) \le \frac12(1 - \eps)\right] &\le \exp\left(-\frac{\eps^2r}{6}\right)
\end{align*}
Fix $\eps = \frac{e^\alpha - 1}{e^\alpha + 1}$. Observe that if $\bap_v(H) \le \frac12(1 + \eps)$ and $\bap'_v(H) \ge \frac12(1 - \eps)$ then
\begin{equation}
\frac{\bap_v(H)}{\bap'_v(H)} \le \frac{1 + \eps}{1 - \eps} = e^\alpha. \label{eq:ealpha}
\end{equation}
Continuing from Eq.~\eqref{eq:prob}
\begin{align*}
\Pr\left[\exists v : \frac{\bap_v(H)}{\bap'_v(H)} > e^\alpha\right] &\le \sum_v \Pr\left[\frac{\bap_v(H)}{\bap'_v(H)} > e^\alpha\right]\\
&\le \sum_v \Pr\left[\bap_v(H) \ge \frac12(1 + \eps) \vee \bap'_v(H) \le \frac12(1 - \eps)\right] & \because \textrm{Eq.~\eqref{eq:ealpha}}\\
&\le \sum_v 2\exp\left(-\frac{\eps^2r}{6}\right) & \because \textrm{Chernoff bound}\\
&= 4 \exp\left(-\frac{\eps^2r}{6}\right)\\
&\le \beta
\end{align*}
where the last line follows from $r = 6\left(\frac{e^\alpha + 1}{e^\alpha - 1}\right)^2 \log \frac{4}{\beta} = \frac{6}{\eps^2}\log \frac{4}{\beta}$.

\section{Proof of Theorem \ref{thm:second_error}}

Proving the sample complexity guarantee is straightforward: Since $\E\left[N_j\right] = m$ for each group $j$, the expected number of users who participate in the mechanism is $\E\left[\sum_{j=1}^g N_j\right] = gm = n$.

To prove the error guarantee, we need the following proposition.

\begin{prop} \label{prop:main} Recall the definitions of $a, b, m, r$ in Mechanism \ref{alg:second}. If
\[
\sigma^2 = \frac{3r^2}{bm^3} + \frac{20r^2}{bm^2} + \frac{16r}{bm}C(\bp) + \frac{2}{b}C(\bp)^2
\]
then the estimate $\hC$ returned by Mechanism \ref{alg:second} satisfies
\[
|\hC - C(\bp)| \le 2\sigma
\]
with probability at least $1 - \exp(-\frac{a}{8})$.\end{prop}
\begin{proof} Fix group $j$ and let $M = N_j$. Let $M_{s, x}$ be the number of users in group $j$ who select salt $s$ and sample $x$. Therefore $M = \sum_{s, x} M_{s, x}$. Let $q_{s, x} = \frac{p_x}{r}$ be the probability that a user selects salt $s$ and sample $x$. Let $\bM$ and $\bq$ be vectors whose components are the $M_{s, x}$'s and $q_{s, x}$'s, respectively. Observe that $\bM$ is a sample from $\textrm{Multinomial}(M, \bq)$. Since $M$ is $\textrm{Poisson}(m)$ distributed, we have by the `Poissonization trick' that each $M_{s, x}$ is independent and $\textrm{Poisson}(mq_{s,x})$ distributed. It is well-known \citep{Riordan37} that if $Z$ is $\textrm{Poisson}(\lambda)$ distributed then \begin{align*}
\E[Z^2] &= \lambda + \lambda^2\\
\E[Z^4] &= \lambda + 7 \lambda^2 + 6 \lambda^3 + \lambda^4
\end{align*}
With a slight abuse of notation, let $C(\bz) = \sum_i z^2_i$ for any non-negative vector $\bz$. Let $V = V_j$. In their classic paper on sketches for data streams, \citet{ALON1999137} analyzed $V^2$ as an estimator for $C(\bM)$. In particular they showed
\begin{align*}
    \E[V^2 ~|~ \bM] &= C(\bM)\\
    \Var[V^2 ~|~ \bM] &\le 2C(\bM)^2.
\end{align*}
We will extend these results to express $\E[V^2]$ and $\Var[V^2]$ in terms of $C(\bp)$. We have
\[
\E[V^2] = \E[\E[V^2 ~|~ \bM]] = \E[C(\bM)].
\]
We express $\E[C(\bM)]$ as
\begin{align*}
\E[C(\bM)] &= \sum_{s, x} \E[M^2_{s, x}]\\
&= \sum_{s, x} mq_{s, x} + m^2q^2_{s, x}\\
&= m + m^2 C(\bq)\\
&= m + \frac{m^2}{r} C(\bp)
\end{align*}
where the last line used
\[
C(\bq) = \sum_{s, x} q^2_{s, x} = \frac{1}{r^2} \sum_{s,x} p^2_x = \frac1r C(\bp).
\]
Also, by the law of total variance
\begin{align*}
\Var[V^2] &= \E[\Var[V^2 ~|~ \bM]] + \Var[\E[V^2 ~|~ \bM]]\\
&\le 2\E[C(\bM)^2] + \Var[C(\bM)]\\
&= 2\E[C(\bM)^2] + \E[C(\bM)^2] - \E[C(\bM)]^2\\
&= 3\E[C(\bM)^2] - \E[C(\bM)]^2.
\end{align*}
Both $\E[C(\bM)]^2$ and $\E[C(\bM)^2]$ can be bounded in terms of $C(\bp)$. From above we conclude
\[
\E[C(\bM)]^2 = m^2 + \frac{2m^3}{r}C(\bp) + \frac{m^4}{r^2}C(\bp)^2.
\]
We express $\E[C(\bM)^2]$ as
\begin{align}
\E[C(\bM)^2] &= \E\left[\left(\sum_{s, x} M^2_{s, x}\right)^2\right] \notag\\
&= \E\left[\sum_{s,x,s',x'} M^2_{s, x}M^2_{s', x'}\right] \notag\\
&= \E\left[\sum_{s, x} M^4_{s, x} + \sum_{s \neq s' \vee x \neq x'} M^2_{s, x}M^2_{s', x'}\right] \notag\\
&= \sum_{s, x} \E[M^4_{s, x}] + \sum_{s \neq s' \vee x \neq x'} \E[M^2_{s, x}]\E[M^2_{s', x'}] \notag\\
&= \sum_{s, x} mq_{s, x} + 7m^2q^2_{s, x} + 6m^3q^3_{s, x} + m^4q^4_{s, x} \notag\\
&~~~~+ \sum_{s \neq s' \vee x \neq x'} (mq_{s, x} + m^2q^2_{s, x})(mq_{s', x'} + m^2q^2_{s', x'}) \tag{$\star$} \label{eq:one}
\end{align}
where we used the independence of the $M_{s, x}$'s and the moments of the Poisson distribution. By expanding and rearranging terms we have
\begin{align*}
\eqref{eq:one} &= \sum_{s, x} mq_{s, x} + 7m^2q^2_{s, x} + 6m^3q^3_{s, x} + m^4q^4_{s, x}\\
&~~~~+ \sum_{s \neq s' \vee x \neq x'} m^2q_{s, x}q_{s',x'} + m^3q^2_{s, x}q_{s',x'} + m^3 q_{s,x}q^2_{s',x'} + m^4q^2_{s,x}q^2_{s', x'}\\
&\le \sum_{s, x} mq_{s, x} + 7m^2q^2_{s, x} + 6m^3q^3_{s, x} + m^4q^4_{s, x}\\
&~~~~+ \sum_{s \neq s' \vee x \neq x'} 7m^2q_{s, x}q_{s',x'} + 3m^3q^2_{s, x}q_{s',x'} + 3m^3 q_{s,x}q^2_{s',x'} + m^4q^2_{s,x}q^2_{s', x'}\\
&= m + \sum_{s, x, s', x'} 7m^2q_{s, x}q_{s',x'} + 3m^3q^2_{s, x}q_{s',x'} + 3m^3q_{s, x}q^2_{s',x'} + m^4q^2_{s,x}q^2_{s', x'}\\
&= m + 7m^2 + 6m^3C(\bq) + m^4 C(\bq)^2\\
&= m + 7m^2 + \frac{6m^3}{r}C(\bp) + \frac{m^4}{r^2}C(\bp)^2.
\end{align*}
Therefore
\[
\Var[V^2] \le 3\E[C(\bM)^2] - \E[C(\bM)]^2 \le 3m + 20m^2 + \frac{16m^3}{r}C(\bp) + \frac{2m^4}{r^2}C(\bp)^2.
\]
Putting everything together, we have for each group $j$
\begin{align*}
\E[C_j] &= \E\left[\frac{r(V^2_j - m)}{m^2}\right] = C(\bp)\\
\Var[C_j] &= \Var\left[\frac{r(V^2_j - m)}{m^2}\right] \le \frac{r^2}{m^4}\Var[V^2_j] = \frac{3r^2}{m^3} + \frac{20r^2}{m^2} + \frac{16r}{m}C(\bp) + 2C(\bp)^2.
\end{align*}
Observe that the $C_j$'s are all independent, because they are calculated from disjoint samples and salts, and also because a group index is prepended to each input to the hash function, causing the hash values to be independent across groups. Therefore for each supergroup $\ell$
\begin{align*}
\E[\baC_\ell] &= \frac1b \sum_{j \in J_\ell} \E[C_j] = C(\bp)\\
\Var[\baC_\ell] &= \frac{1}{b^2} \sum_{j \in J_\ell} \Var[C_j] \le \frac{3r^2}{bm^3} + \frac{20r^2}{bm^2} + \frac{16r}{bm}C(\bp) + \frac{2}{b}C(\bp)^2 = \sigma^2
\end{align*}
where the last equality follows from the definition of $\sigma^2$. We know by the analysis of the median-of-means estimator \citep{lugosi2019mean} that 
\[
\Pr[|\hC - C(\bp)| \ge 2\sigma] \le \exp\left(-\frac{a}{8}\right)
\]
which proves the proposition.\end{proof}

We are now ready to complete the proof of the error guarantee. From Proposition \ref{prop:main} we have
\[
\sigma^2 = \frac{3r^2}{bm^3} + \frac{20r^2}{bm^2} + \frac{16r}{bm}C(\bp) + \frac{2}{b}C(\bp)^2 \le \frac{23r^2}{bm^2} + \frac{16r}{bm}C(\bp) + \frac{2}{b}C(\bp)^2.
\]
Plugging $m = \frac{n}{ab}$ and $b = \frac{20}{\eps^2_{\textrm{rel}}}$ into the previous inequality yields
\[
\sigma^2 \le \frac{460a^2r^2}{n^2\eps^2_{\textrm{rel}}} + \frac{16ar}{n}C(\bp) + \frac{\eps^2_{\textrm{rel}}}{10}C(\bp)^2. 
\]
Plugging $n \ge \frac{1280r\log \frac1\delta}{\eps^2_{\textrm{rel}} C(\bp)} = \frac{160ar}{\eps^2_{\textrm{rel}} C(\bp)}$ into the previous inequality yields
\[
\sigma^2 \le \frac{460\eps^2_{\textrm{rel}}}{160^2} C(\bp)^2 + \frac{\eps^2_{\textrm{rel}}}{10}C(\bp)^2 + \frac{\eps^2_{\textrm{rel}}}{10} C(\bp)^2 = \left(\frac{460}{160^2} + \frac{2}{10}\right)\eps^2_{\textrm{rel}} C(\bp)^2.
\]
Since $\exp(-\frac{a}{8}) = \delta$ we have by Proposition \ref{prop:main}
\[
|\hC - C(\bp)| \le 2\sigma \le \left(2\sqrt{\frac{460}{160^2} + \frac{2}{10}}\right)\eps_{\textrm{rel}} C(\bp) < \eps_{\textrm{rel}} C(\bp)
\]
with probability at least $1 - \delta$.

\section{Proof of Corollary \ref{cor:second_error}}

If $\alpha \le 1$ then $\frac{e^\alpha + 1}{e^\alpha - 1} \le O\left(\frac{1}{\alpha}\right)$ because $e^\alpha + 1 \le O(1)$ for all $\alpha \le 1$ and $1 + \alpha \le e^\alpha$ for all $\alpha \in \bbR$. Also let $\eps_{\textrm{rel}} = \frac{\eps}{C(\bp)}$ in the statement of Theorem \ref{thm:second_error}. 

\section{Proof of Theorem \ref{thm:private_lower} }
\label{app:private_lower}

We make use of the lower bound for local differential privacy introduced by~\cite{DuchiWJ16} which relies on a privatized version of Le Cam's two point method. Accordingly, we construct a pair of problem instances $\mathbf{p}_0$ and $\mathbf{p}_1$ for which $ d_C(\mathbf{p}_0, \mathbf{p}_1)= \vert C(\mathbf{p}_0) - C(\mathbf{p}_1) \vert  \ge \Omega ( \tau ) $ and at the same time $d_{\text{KL}}(\mathbf{p}_0, \mathbf{p}_1 ) \in \Theta ( \tau^2)$. Specifically, let
\begin{align}
    \mathbf{p}_0  = \left(\frac{1}{2(K-1)}, \dots, \frac{1}{2(K-1)}, \frac{1}{2} \right) \label{eq:batch_lower1} \quad \text{and} \quad
    \mathbf{p}_1  = \left(\frac{1-\tau}{2(K-1)}, \dots, \frac{1-\tau}{2(K-1)}, \frac{1+\tau}{2} \right) 
\end{align}
The KL divergence between $\mathbf{p}_0$ and $\mathbf{p}_1$ is
\[
d_{\text{KL}} \left( \mathbf{p}_0, \mathbf{p}_1 \right) = \frac{1}{2}\log \frac{1}{1-\tau^2} = \Theta(\tau^2)
\]
and the absolute difference between their collision probability is
\begin{align*}
d_C(\mathbf{p}_0, \mathbf{p}_1) = \left\vert C(\mathbf{p}_0) - C(\mathbf{p}_1) \right\vert 
    & = \frac{\tau}{2} \left( 1+ \left(\frac{\tau}{2} - \frac{1}{2(K-1)} \right) \right) \ge \tau /2
\end{align*}
Proposition 1 of \citet{DuchiWJ16} immediately implies the following Corollary.
\begin{cor}\label{corr:ldp_lower}
Let $\theta$ be an estimator of $C(\mathbf{p})$ which receives $n$ observations from an $\alpha$-locally differential private channel $Q$ with $\alpha \in [0, 23/35]$, \emph{i.e.}, channel $Q$ is a conditional probability distribution which maps each observation $x_i$ to a probability distribution on some finite discrete domain $\mathcal{Z}$. We will denote the privatized data by $Z_i \sim Q( \cdot \vert x_i)$. Then for any pair of distributions $\mathbf{p}_0$ and $\mathbf{p}_1$ such that $d_C(\mathbf{p}_0, \mathbf{p}_1 ) \ge \tau / 2 $, it holds that
\[
\inf_{Q} \inf_{\theta} \sup_{\mathbf{p}} \mathbb{E}_{Q,\mathbf{p}} \left[ d_C( \mathbf{p}, \theta ( Z_1, \dots, Z_n)) \right] \ge \frac{\tau}{4} \left(1 - \sqrt{2\alpha^2 n d_{\text{KL}} \left( \mathbf{p}_0, \mathbf{p}_1 \right)} \right)
\]
\end{cor}
Corollary \ref{corr:ldp_lower} applied to the the pair of distribution defined in \eqref{eq:batch_lower1} with $\tau = 1/(\alpha \sqrt{n})$ implies  that Mechanism 1 is minimax optimal in terms of $\epsilon$ and $\alpha$ by achieving a sample complexity that is $O(1/(\alpha \sqrt{n}))$. 

\section{Proof of Theorem \ref{thm:failure}}

In the simplest version of the $k$-RAPPOR algorithm \citep{erlingsson2014rappor}, each user $i$ with private value $x_i \in [k]$ first constructs a vector $\bv_i \in \{0, 1\}^k$ that has $1$ in component $x_i$ and $0$ elsewhere, and then reports the noisy vector $\hbv_i$ formed by independently flipping each bit in $\bv_i$ with probability $\frac{1}{e^{\alpha/2} + 1}$. Since $x_i$ is drawn from $\bp$, the probability that component $x$ of $\hbv_i$ is $1$ is equal to $q_x$, where
\[
q_x = p_x \cdot \frac{e^{\alpha/2}}{e^{\alpha/2} + 1} + (1 - p_x) \cdot \frac{1}{e^{\alpha/2} + 1}.
\]
The distribution $\tbp = A(x_1, \ldots, x_n)$ estimated by $k$-RAPPOR is defined by $\tp_x = a\hp_x - b$ for all $x \in [k]$, where $\hbp$ is the empirical average of the $\hbv_i$'s, and $a = \frac{e^{\alpha/2} + 1}{e^{\alpha/2} - 1}$ and $b = \frac{1}{e^{\alpha/2} - 1}$ serve to debias the noise. In other words, $\E[\tp_x] = p_x$ for all $x \in [k]$. For all $\alpha \in (0, 1)$ this version of the $k$-RAPPOR algorithm is $(\alpha, 0)$-minimax optimal \citep{acharya2019hadamard}. Let $\bp$ be the uniform distribution. Let $\eps_x = p_x - \tp_x$. We have
\begin{align*}
\E[|C(\tbp) - C(\bp)|] &= \E\left[\left|\sum_x \tp^2_x - \sum_x p^2_x\right|\right]\\
&= \E\left[\left|\sum_x (p_x - \eps_x)^2 - \sum_x p^2_x\right|\right]\\
&= \E\left[\left|\sum_x \eps^2_x - 2 \sum_x \eps_x p_x\right|\right]\\
&\ge \E\left[\left|\sum_x \eps^2_x\right|\right] - \E\left[\left|2\sum_x p_x \eps_x\right|\right] & \because \textrm{Triangle inequality}\\
&= \sum_x \E\left[\eps^2_x\right] - 2\E\left[\left|\sum_x p_x \eps_x\right|\right]\\
&\ge \sum_x \E\left[\eps^2_x\right] - 2\E\left[\sqrt{\sum_x p^2_x}\sqrt{\sum_x \eps^2_x}\right] & \because \textrm{Cauchy-Schwarz}\\
&= \sum_x \E\left[\eps^2_x\right] - \frac{2}{\sqrt{k}}\E\left[\sqrt{\sum_x \eps^2_x}\right] & \because \bp\textrm{ is uniform}\\
&\ge \sum_x \E\left[\eps^2_x\right] - \frac{2}{\sqrt{k}}\sqrt{\sum_x \E\left[\eps^2_x\right]} & \because \textrm{Jensen's inequality}\\
&= \sum_x \Var\left[\tp_x\right] - \frac{2}{\sqrt{k}}\sqrt{\sum_x \Var\left[\tp_x\right]} & \because \textrm{Definition of }\eps_x \textrm{ and }\E[\tp_x] = p_x\\
&= a^2\sum_x \Var\left[\hp_x\right] - \frac{2a}{\sqrt{k}}\sqrt{\sum_x \Var\left[\hp_x\right]} & \because \textrm{Definition of }\tp_x\\
&= \frac{a^2}{n} \sum_x q_x(1-q_x) - \frac{2a}{\sqrt{kn}}\sqrt{\sum_x q_x(1-q_x)} & \because \hp_x \textrm{ is average of }n\textrm{ independent samples}
\end{align*}
It is clear from the definition of $q_x$ that
\[
\frac{1}{e^{\alpha/2} + 1} \le q_x \le \frac{e^{\alpha/2}}{e^{\alpha/2} + 1}.
\]
We also have $\frac{e^{\alpha/2}}{e^{\alpha/2} + 1} \le \frac23$ and $\frac{1}{e^{\alpha/2} + 1} \ge \frac13$, since $\alpha \in (0, 1)$. Therefore, continuing from above, we have
\begin{align*}
\E[|C(\tbp) - C(\bp)|] &\ge \frac{a^2}{n} \sum_x q_x(1-q_x) - \frac{2a}{\sqrt{kn}}\sqrt{\sum_x q_x(1-q_x)} & \textrm{From above}\\
&\ge \frac{a^2}{n} \cdot k \cdot \frac{1}{e^{\alpha/2} + 1}\left(1 - \frac{e^{\alpha/2}}{e^{\alpha/2} + 1}\right) - \frac{2a}{\sqrt{kn}}\sqrt{k \cdot \frac{e^{\alpha/2}}{e^{\alpha/2} + 1}\left(1 - \frac{1}{e^{\alpha/2} + 1}\right)} & \because \textrm{Bounds on }q_x\\
&\ge \frac{a^2k}{9n} - \frac{4a}{3\sqrt{n}} & \because \alpha \in (0, 1)\\
&\ge \frac{4k}{9 \alpha^2 n} - \frac{8}{\alpha\sqrt{n}} & \because \textrm{Definition of }a
\end{align*}
where the last inequality follows because $1 + z \le e^z \le 1 + 2z$ and $2 \le e^z + 1 \le 3$ for all $z \in (0, \frac12)$. Therefore if $\E[|C(\tbp) - C(\bp)|] \le \eps$ we have
\[
n \ge \frac{1}{\eps}\left(\frac{4k}{9\alpha^2} - \frac{8\sqrt{n}}{\alpha}\right).
\]
If $n > \frac{k^2}{1296\alpha^2}$ we are done. Otherwise $n \le \frac{k^2}{1296\alpha^2}$, which implies $\sqrt{n} \le \frac{k}{36\alpha}$. Replacing $\sqrt{n}$ in the above expression with this upper bound shows that $n \ge \frac{2k}{9\alpha^2\eps}$, which completes the proof.

\section{Proof of Theorem \ref{thm:general_closeness_thm_simple}}

First let us define
\begin{align}\label{eq:ustat_seqtesting}
U_m = \frac{2}{m(m-1)} \sum_{i=1}^m \sum_{j=1}^{i-1} \indic{X_i = X_j}
\end{align}
based on $\{ X_1, \dots, X_m\}$. The sequence $U_1, U_2,, \dots $ are dependent sequences, since each of them depends on all previous observations, thus we shall apply a decoupling technique to obtain a martingale sequence which we can use in a sequential test. Based on $U_m$, let us define
\[
\bar{U}_m:= \sum_{i=1}^m \sum_{j=1}^{i-1} g_{\mathbf{p}}(X_i, X_j)
\]
with 
\begin{align*}
g_{\mathbf{p}}(X_i, X_j) 
    & = 
    \indic{X_i = X_j} - \Pr \left( X_i = X_j \vert X_i \right) - \Pr \left( X_i = X_j \vert X_j \right) + C (\mathbf{p})
    \enspace .
\end{align*}
This decoupling technique is motivated by Theorem 8.1.1 of~\citet{de1999decoupling}, since the kernel function $g$ has became centered and degenerate, \emph{i.e.}, $\E \left[ g_{\mathbf{p}}(X_i, X_j) \vert X_j \right] = \E \left[ g_{\mathbf{p}}(X_i, X_j) \vert X_i \right] = 0$, which implies that $\bar{U}_n$ is a zero-mean martingale with $n\ge 2$ as follows.
\begin{lem}\label{lem:decoupling}
$\bar{U}_2, \bar{U}_3, \dots$ is a discrete-time martingale the filtration of which is defined $\mathcal{F}_t = \{ X_1, \dots, X_m \}$ and for all $m$, $$\E \left[ Y_m(\mathbf{p}) \vert \mathcal{F}_{m-1}\right] = 0$$
where $Y_m(\mathbf{p}) = \sum_{i=1}^{m-1} g_{\mathbf{p}}(X_m, X_i)$ if $m\ge 2$ and $Y_1=0$.
\end{lem}
\begin{proof}
This decoupling is motivated by Theorem 8.1.1 of \citet{de1999decoupling}. First note that $\E \left[ g_{\mathbf{p}}(X_i, X_j) \vert X_j \right] = \E \left[ g_{\mathbf{p}}(X_i, X_j) \vert X_i \right] = 0$ by construction. This implies that
\[
\E \left[ Y_m(\mathbf{p}) \vert \mathcal{F}_{m-1} \right] = \sum_{i=1}^{m-1} \E\left[g_{\mathbf{p}}(X_m, X_i) \vert \mathcal{F}_{m-1} \right] = \sum_{i=1}^{m-1} \E\left[g_{\mathbf{p}}(X_m, X_i) \vert X_i \right] = 0
\]
for all $m\ge 2$. Since $\bar{U}_m = \sum_{i=1}^m  Y_i ( \mathbf{p} )$, it holds that
\[
\E \left[ \bar{U}_m \vert \mathcal{F}_{m-1} \right] = \bar{U}_{m-1} + \E \left[ Y_m (\mathbf{p}) \vert \mathcal{F}_{m-1} \right] = \bar{U}_{m-1} \enspace .
\]
Finally, it is straightforward that $\E \left[ \vert Y_m (\mathbf{p}) \vert \right] < \infty $ which implies that $\bar{U}_2, \bar{U}_3, \dots$ is a discrete-time martingale by definition.
\end{proof}
The empirical sequence is $\bar{u}_m = \sum_{i=1}^m y_m (\mathbf{p})$ with
\[
y_j(\mathbf{p}) = \sum_{i=1}^{m-1} \indic{x_i = x_j} - \sum_{i=1}^{m-1} p_{x_i} - (m-1) p_{x_j} + (m-1) C(\mathbf{p})
\]
which is a realization of a martingale with bounded difference such that $\vert \bar{U}_k - \bar{U}_{k-1}\vert = \vert Y_k\vert \le 4m$ and $y_1(\mathbf{p})=0$.
However we cannot compute the empirical sequence, since the parameters of distribution are not known. As a remedy, we further decompose $\bar{U}_n$ as the sum of two sequences based on the observation that
\[
\mathbb{E} \left[ p_{X_i} \right] = \sum_{i} p_{x_i}^2 = C(\mathbf{p})
\]
which implies that
$
\sum_{i=1}^{m} ( p_{X_i} - C(\mathbf{p}) )
$
which is again a zero-mean martingale sequence with the same filtration $\mathcal{F}_m$ such that the difference $\vert p_{X_i} - C(X) \vert < 1$ for all $i$. This motivates the following decomposition of $\bar{U}_n$ as
\[
Y_j(\mathbf{p}) = \underbrace{\sum_{i=1}^{j-1} \indic{X_i = X_j} - 2 (j-1) C(\mathbf{p})}_{T_j(\mathbf{p})} + \underbrace{2(j-1) C(\mathbf{p}) - \sum_{i=1}^{j-1} p_{X_i} - (j-1) p_{X_j}}_{E_j(\mathbf{p})}
\]
Note that $T_m (\mathbf{p})$ can be computed, and it is a zero-mean martingale sequence up to an error term $E_n (\mathbf{p})$ which we cannot be computed, since the parameters of the underlying distribution $\mathbf{p}$ is not available to the tester. Also note that $T_m (\mathbf{p})$ is a centralized version of $U_m$ defined in \eqref{eq:ustat_seqtesting}. More detailed, we have that
\[
\frac{2}{m(m-1)} \sum_{i=1}^m T_m (\mathbf{p}) = U_m - C(\mathbf{p})
\]
which means that Algorithm \ref{alg:closeness_simple} uses the sequence of $U_1, \dots, U_m$ as test statistic which was our point of departure. Now we will focus on $E_m (\mathbf{p})$ and how it can be upper bounded.

Further note that $E_n (\mathbf{p})$ can be again decomposed into sequence of sums of zero mean-mean terms which we can upper bound with high probability. Due to the construction, it holds that
\begin{align*}
    \sum_{i=1}^m  Y_i(\mathbf{p}) & =  \sum_{i=1}^m  T_i(\mathbf{p}) +  \sum_{i=1}^m E_i (\mathbf{p} )
\end{align*}
We can apply the time uniform confidence interval of ~\citet{HoRaMcSe21} to the lhs which implies that it holds that
\begin{align}\label{eq:first_term}
\Pr \left[ \forall m \in \mathbb{N} : \left\vert \frac{2}{m(m-1)} \sum_{i-1}^m Y_i (\mathbf{p} ) \right\vert \ge \phi( i, \delta ) \right] \le \delta \enspace .
\end{align}
if the data is generated from $\mathbf{p}$ where 
\[
\phi( i, \delta ) = 1.7\sqrt{\frac{\log \log i + 0.72 \log (10.4 /\delta)}{i}} \enspace .
\] 
Note that the confidence interval of \citet{HoRaMcSe21} applies to the sum of discrete time martingales where each term is sub-Gaussian. This also applies to  $Y_i(\mathbf{p}$, since it is a bounded random variable.

Next we upper bound $\sum_i E_i(\mathbf{p})$. For doing so, we decompose each term as
\[
E_i (\mathbf{p} ) = \sum_{j=1}^{i-1} \big( F_2(\mathbf{p}) - p_{X_j} \big) + (i-1)(F_2(\mathbf{p}) - p_{X_i}) 
\]
which implies
\begin{align*}
    \sum_{i-1}^m E_i (\mathbf{p} )
    & = 
    \sum_{i=1}^m \sum_{j=1}^{i-1} \big( F_2(\mathbf{p}) - p_{X_j} \big) + \sum_{i=1}^m (i-1)(F_2(\mathbf{p}) - p_{X_i}) \\    
    & = 
    \sum_{i=1}^{m-1} (m-i) \big( F_2(\mathbf{p}) - p_{X_i} \big) + \sum_{i=1}^m (i-1)(F_2(\mathbf{p}) - p_{X_i}) \\
    & = m \sum_{i=1}^{m} \big( F_2(\mathbf{p}) - p_{X_i} \big)
\end{align*}
Apply the time uniform confidence interval of ~\citet{HoRaMcSe21} to $E_i (\mathbf{p} )$, we have that
\begin{align} \label{eq:second_term}
\Pr \left[ \forall m \in \mathbb{N} : \left\vert \frac{2}{m(m-1)} \sum_{i-1}^m E_i (\mathbf{p} ) \right\vert \ge \phi( i, \delta ) \right] \le \delta \enspace .
\end{align}
Due to union bound, we can upper bound the difference of $T_i (\mathbf{p})$ and $Y_i(\mathbf{p} )$ using \eqref{eq:second_term} and \eqref{eq:first_term} as
\[
\frac{2}{m(m-1)}\left\vert \sum_{i=1}^m  Y_i(\mathbf{p}) -  \sum_{i=1}^m  T_i(\mathbf{p}) \right\vert \le 2\phi( i, \delta/2 )
\]
with probability at least $1-\delta$ for all $m$ even if $m$ is a random variable that depends on $X_1, \dots, X_m$. This implies that if the observations is generated from a distribution with parameters $\mathbf{p}$, then $\tfrac{2}{m(m-1)}T_i (\mathbf{p})$ stays close to zero, including all distribution $\mathbf{p}_0$ such that $C(\mathbf{p}_0) = c_0$. This implies the correctness of Algorithm ~\ref{alg:closeness_simple}.

Finally note that 
\begin{align*}
    \left\vert \frac{2}{m(m-1)} \sum_{i=1}^m Y_m (\mathbf{p}) - \frac{2}{m(m-1)} \sum_{i=1}^m Y_i(\mathbf{p_0}) \right\vert = \vert C(\mathbf{p} ) - \underbrace{C( \mathbf{p}_0 )}_{=c_0} \vert
\end{align*}
for any $\mathbf{p}_0$ such that $C(\mathbf{p}_0) = c_0$ which implies the sample complexity bound. This concludes the proof.

\section{Proof of Theorem \ref{corr:lower_seq}}

Before we proof the lower bound, we need to get a better understating of the relation of the total variation distance and $d_C( \mathbf{p}, \mathbf{p}')= \vert C(\mathbf{p}) - C(\mathbf{p}')\vert $

\subsection{Total variation distance}

The \emph{total variation distance} between random variables $X$ and $Y$ is defined
\[
\left \vert X - Y \right \rvert = \frac12 \sum_z \left \lvert\Pr[X = z] - \Pr[Y = z]\right \rvert
\]
where the sum is over the union of the supports of $X$ and $Y$. 
\begin{thm} \label{thm:reduction} For any $X$ and $Y$
\[
\left \lvert C(X) - C(Y) \right \rvert \le 6 \left \lvert X - Y \right \rvert.
\]\end{thm}

\begin{proof}
Let $z_1, z_2, \ldots$ be an enumeration of the union of the supports of $X$ and $Y$. Let $p_i = \Pr[X = z_i]$ and $q_i = \Pr[Y = z_i]$.

Assume without loss of generality $C(X) \le C(Y)$. It suffices to prove $C(Y) \le C(X) + 6|X - Y|$. Let $\delta_i = p_i - q_i$. We have
\begin{align*}
C(X) &= \sum_i p^2_i\\
&= \sum_i (q_i + \delta_i)^2\\
&= \sum_i q^2_i + 2\sum_i q_i \delta_i + \sum_i \delta^2_i\\
&\le \sum_i q^2_i + 2\sum_i q_i |\delta_i| + \sum_i \delta^2_i\\
&\le \sum_i q^2_i + 2\sum_i |\delta_i| + \sum_i \delta^2_i\\
&\le \sum_i q^2_i + 2\sum_i |\delta_i| + \sum_i |\delta_i|\\
&= C(Y) + 6|X - Y|
\end{align*}
and rearranging completes the proof.
\end{proof}

\subsection{Lower bound}

Based on of Lemma A.1 due to~\cite{AaFlGa21}, one can lower bound the stopping time of any sequential testing algorithm in expectation. Note that this lower bound readily applies to our setup and implies a lower bound for the expected sample complexity which is
\begin{align}\label{eq:lower}
\frac{\log 1/3\delta}{ d_{\text{KL}} (\mathbf{p}, \mathbf{p}')}
\end{align}
where
\[
d_{C }(\mathbf{p}, \mathbf{p}') = \vert  C(\mathbf{p}) - C(\mathbf{p}') \vert = \epsilon
\]

In addition to this, the following Lemma lower bounds the sensitivity of KL divergence in terms of collision probability.
\begin{lem}\label{lem:dc_dkl}
For any random variables $X$ and $X'$ with parameters $\mathbf{p}$ and $\mathbf{p}$, it holds
\[
d_C( \mathbf{p}, \mathbf{p}')^2 \le 18d_{\text{KL}}( \mathbf{p}, \mathbf{p}')
\]
\end{lem}
\begin{proof}
Pinsker's inequality and  Theorem \ref{thm:reduction} implies this result.
\end{proof}

Lemma \ref{lem:dc_dkl} applied to \eqref{eq:lower} implies that Theorem \ref{thm:general_closeness_thm_simple} achieves optimal sample complexity, since for any distribution for which 
\[
C( \mathbf{p}_0) = c_0
\]
and 
\[
d_{C }(\mathbf{p}, \mathbf{p}') = \epsilon
\]
the expected sample complexity of any tester is lower bounded by
\[
\frac{\log 1/3\delta}{\epsilon^2}
\]
This concludes the proof.


\section{Batch testers used in the experiments}

In the experimental study, we used two batch testers as baseline. Each of these testers are based on learning algorithm which means that using a learning algorithm, the collision probability is estimated with an additive error $\epsilon /2 $ and then one can decide whether the true collision probability is close to $c_0$ or not. This approach is caller \emph{testing-by-learning}. In this section, we present exact sample complexity bound for these batch testers and in addition to this, we show that these approaches are optimal in minimax sense for testing collision probability for discrete distributions. In this section we present the following results:

\begin{itemize}
    \item We start by presenting a minmax lower bound for the batch testing problem which is $\Omega (\epsilon^{-2})$. In addition, we also show that the same lower bound applies to learning.
    \item In Subsection \ref{sec:warmup}, we consider two estimators, i.e. plug-in and U-statistic, and we compute their sample complexity upper bound that are of order $\epsilon^{-2}$ and they differ only in constant. These are presented in Theorem \ref{thm:plug-in} and \ref{thm:ustat_estimatorsamp}, respectively. 
    \item In Subsection \ref{subseq:testing-by-learning}, we present the testing-by-learning approach and discuss that the plug-in estimator is minmax optimal on a wide range of parameters.
\end{itemize}

\subsection{Lower bound for estimation and testing}

To construct lower bound for estimation and testing we consider the pair of distributions defined in (\ref{eq:batch_lower1}) with $\tau = \epsilon$. In this case, we obtain two distributions such that $d_{\text{KL}} \left( \mathbf{p}_0, \mathbf{p}_1 \right) = \Theta(\epsilon^2)$ and $d_{C} \left( \mathbf{p}_0, \mathbf{p}_1 \right) \ge \epsilon/2$. Then estimator lower bound can be obtained based on LeCam's theorem (See Appendix \ref{app:lecam}) which is $\Theta(1/\epsilon^2)$ as follows.
\begin{cor}
For any estimator $\hat{\theta}_n$ for Collision probability $F_2 (\mathbf{p})$ based on $n\in o(1/\epsilon^2)$, there exist a discrete distribution $\mathbf{p}$ for which 
\[
\mathbb{E}_P \left[ \left\vert \hat{\theta}_n ( \mathcal{D}_n )- F_2(\mathbf{p}) \right\vert \right] \ge C \cdot \epsilon
\]
where $C>0$ does not depend on the distribution $\mathbf{p}$.
\end{cor}

One can show a similar lower bound for testing using Neyman-Pearson lemma. We refer the reader to Section 3.1 of \citet{canonne2022topics} for more detail. We recall this result here with $d_C$.
\begin{cor}
Let $f$ an $(\epsilon, \delta)$-tester with sample complexity $n$. Then for any pair of distributions $\mathbf{p}_0$  and $\mathbf{p}_1$ such that $d_C (\mathbf{p}_0, \mathbf{p}_1 ) = \epsilon$, it holds that
\[
1-2\delta \le d_{\text{TV}} ( \mathbf{p}_0^{\otimes n}, \mathbf{p}_1^{\otimes n} )
\]
where $p^{\otimes n}$ is the $n$ times product distribution from $\mathbf{p}$.
\end{cor}
Using Pinsker's inequality it results in that
\[
d_{\text{TV}} ( \mathbf{p}_0^{\otimes n}, \mathbf{p}_1^{\otimes n} )^2 \le \frac{1}{2} d_{\text{KL}} ( \mathbf{p}_0^{\otimes n}, \mathbf{p}_1^{\otimes n} ) = \frac{n}{2} d_{\text{KL}} ( \mathbf{p}_0, \mathbf{p}_1 ) \enspace.
\]
Accordingly, since we already constructed a pair of distributions for which $d_2 (\mathbf{p}_0, \mathbf{p}_1 ) = \epsilon$ and $d_{\text{KL}} \left( \mathbf{p}_0, \mathbf{p}_1 \right) = \Omega (\epsilon^2)$, the sample compelxity lower bound for testing is also $\Omega ( 1/\epsilon^2 )$.

\subsection{Plug-in estimator versus U-statistic estimator}
\label{sec:warmup}

\label{sec:plug-in}

 The first estimator is the plug-in estimator which estimates the distribution $\mathbf{p}$ by the normalized empirical frequencies $\widehat{\mathbf{p}} := \widehat{\mathbf{p}}(\mathcal{D}_m)$ and then the estimator is computed as
\[
C(\widehat{\mathbf{p}}) = F_2( \widehat{\mathbf{p}} ) = \sum_{i=1}^K \widehat{p}_{i}^2
\]
In this section, we will other frequency moments of discrete distributions, therefore we will use $F_k( \mathbf{p})$ as the frequency moment of order $k$, which is the collision probability with $k=2$.

The plug-in estimator is well-understood in the general case via lower and upper bound that are presented in ~\cite{theertha2014complexity}. Here we recall an additive error bound under Poissonization which assumes that the sample size is chosen as $M \sim \text{Poi}(m)$ and the data is then $\mathcal{D}_M$.

\begin{thm}\label{thm:plug-in}
If 
\[
m \ge \max\left\{ \frac{1600 F_{3/2}(\mathbf{p})^2}{\epsilon^2}, \frac{8}{\epsilon^2}\log \frac{2}{\delta} \right\} = \frac{8}{\epsilon^2} \cdot \max \left\{ 200 \cdot F_{3/2}(\mathbf{p})^2, \log \frac{2}{\delta} \right\} \enspace .
\]
then 
\[
\mathbb{P}\left( \vert F_2( \widehat{\mathbf{p}}(\mathcal{D}_M )) - T_2(X)  \vert \ge \epsilon  \right) \le \delta
\]
where the dataset $\mathcal{D}_M$ is sampled with sample size $M \sim \text{Poi}(m)$.
\end{thm}
\begin{proof}
Based on Theorem 9 of \cite{theertha2014complexity}, the bias of the estimator with Poissonization is
\[
\left\vert \mathbb{E} \left[ F_2( \widehat{\mathbf{p}}(\mathcal{D}_M )) \right] - T_2(X) \right\vert \le \frac{8}{m} + \frac{10}{\sqrt{m}} F_{3/2} ( \mathbf{p} )
\]
and its variance is 
\[
\mathbb{V} \left[ F_2( \widehat{\mathbf{p}}(\mathcal{D}_M )) \right] \le \frac{64}{m^3} + \frac{4 \cdot 17}{\sqrt{m}} F_{7/2} ( \mathbf{p} ) \enspace .
\]
Thus
\begin{align*}
    \mathbb{P}\left( \left\vert F_2( \widehat{\mathbf{p}}(\mathcal{D}_M )) - T_2(X) \right\vert \ge \epsilon  \right) 
    & =
    \mathbb{P}\left( \left\vert F_2( \widehat{\mathbf{p}}(\mathcal{D}_M )) - \mathbb{E} \left[ F_2( \widehat{\mathbf{p}}(\mathcal{D}_M )) \right]  +  \mathbb{E} \left[ F_2( \widehat{\mathbf{p}}(\mathcal{D}_M )) \right] - T_2(X) \right\vert  \ge \epsilon \right) \\
    & \le 
    \mathbb{P}\left( \left\vert F_2( \widehat{\mathbf{p}}(\mathcal{D}_M )) - \mathbb{E} \left[ F_2( \widehat{\mathbf{p}}(\mathcal{D}_M )) \right] \right\vert  \ge \epsilon - \left\vert \mathbb{E} \left[ F_2( \widehat{\mathbf{p}}(\mathcal{D}_M )) \right] - T_2(X) \right\vert \right)
\end{align*}
where we applied the triangle inequality. Therefore if $m$ is big enough, then it holds that
\begin{align}\label{eq:plugin_proof_1}
\frac{8}{m} + \frac{10}{\sqrt{m}} F_{3/2} ( \mathbf{p} ) \le \frac{\epsilon}{2}
\end{align}
and also holds
\begin{align}\label{eq:plugin_proof_2}
\mathbb{P}\left( \left\vert F_2( \widehat{\mathbf{p}}(\mathcal{D}_M )) - \mathbb{E} \left[ F_2( \widehat{\mathbf{p}}(\mathcal{D}_M )) \right] \right\vert  \ge \epsilon/2 \right) \le \delta
\end{align} 
thus the statement in the theorem holds. What remains is to compute a lower bound for $m$. \eqref{eq:plugin_proof_1} holds if
\[
m \ge \max\left\{ \frac{32}{\epsilon}, \frac{1600 F_{3/2}(\mathbf{p})^2}{\epsilon^2} \right\} \enspace .
\]
Based on Bernstein's inequality, see Theorem \ref{thm:bernstein} in Appendix \ref{app:tech}, \eqref{eq:plugin_proof_2} hold if
\[
m \ge \max \left\{ \frac{8 \log \tfrac{2}{\delta}}{\epsilon^2}, \frac{4736\cdot \sqrt[4]{\log \tfrac{2}{\delta}}}{\sqrt{\epsilon}}, \frac{6528\sqrt[3]{F_{7/2} (\mathbf{p}) \log \tfrac{2}{\delta}} }{\epsilon^{4/3}} \right\}
\]
which concludes the proof.
To simplify the last terms, alternatively we can apply Hoeffding's inequality which yields that \eqref{eq:plugin_proof_2} holds whenever
\[
m \ge \frac{8}{\epsilon^2}\log \frac{2}{\delta} \enspace .
\]
Finally note that $32/\epsilon \le  8/\epsilon^2 \log 2/\delta$ for any $\epsilon, \delta \in (0,1]$ which concludes the proof.
\end{proof}

\begin{thm}\label{thm:ustat_estimatorsamp}
If 
\[
m \ge \max \left\{\frac{32(F_3(X)- F_2(X)^2)}{\epsilon^2} \ln \frac{4}{\delta}, \frac{128+1/6}{\epsilon} \ln \frac{4}{\delta} \right\}
\]
then
\[
\mathbb{P}\left( \left\vert F_2 (X) - U(\mathcal{D}_m) \right\vert  \ge \epsilon \right) \le \delta
\]
\end{thm}

\begin{proof}
Based on Theorem 9 of \cite{theertha2014complexity}, the bias of the estimator with Poissonization is
\[
\left\vert \mathbb{E} \left[ F_2( \widehat{\mathbf{p}}(\mathcal{D}_M )) \right] - T_2(X) \right\vert \le \frac{8}{m} + \frac{10}{\sqrt{m}} F_{3/2} ( \mathbf{p} )
\]
and its variance is 
\[
\mathbb{V} \left[ F_2( \widehat{\mathbf{p}}(\mathcal{D}_M )) \right] \le \frac{64}{m^3} + \frac{4 \cdot 17}{\sqrt{m}} F_{7/2} ( \mathbf{p} ) \enspace .
\]
Thus
\begin{align*}
    \mathbb{P}\left( \left\vert F_2( \widehat{\mathbf{p}}(\mathcal{D}_M )) - T_2(X) \right\vert \ge \epsilon  \right) 
    & =
    \mathbb{P}\left( \left\vert F_2( \widehat{\mathbf{p}}(\mathcal{D}_M )) - \mathbb{E} \left[ F_2( \widehat{\mathbf{p}}(\mathcal{D}_M )) \right]  +  \mathbb{E} \left[ F_2( \widehat{\mathbf{p}}(\mathcal{D}_M )) \right] - T_2(X) \right\vert  \ge \epsilon \right) \\
    & \le 
    \mathbb{P}\left( \left\vert F_2( \widehat{\mathbf{p}}(\mathcal{D}_M )) - \mathbb{E} \left[ F_2( \widehat{\mathbf{p}}(\mathcal{D}_M )) \right] \right\vert  \ge \epsilon - \left\vert \mathbb{E} \left[ F_2( \widehat{\mathbf{p}}(\mathcal{D}_M )) \right] - T_2(X) \right\vert \right)
\end{align*}
where we applied the triangle inequality. Therefore if $m$ is big enough, then it holds that
\begin{align}\label{eq:plugin_proof_1_app}
\frac{8}{m} + \frac{10}{\sqrt{m}} F_{3/2} ( \mathbf{p} ) \le \frac{\epsilon}{2}
\end{align}
and also holds
\begin{align}\label{eq:plugin_proof_2_app}
\mathbb{P}\left( \left\vert F_2( \widehat{\mathbf{p}}(\mathcal{D}_M )) - \mathbb{E} \left[ F_2( \widehat{\mathbf{p}}(\mathcal{D}_M )) \right] \right\vert  \ge \epsilon/2 \right) \le \delta
\end{align} 
thus the statement in the theorem holds. What remains is to compute a lower bound for $m$. \eqref{eq:plugin_proof_1_app} holds if
\[
m \ge \max\left\{ \frac{32}{\epsilon}, \frac{1600 F_{3/2}(\mathbf{p})^2}{\epsilon^2} \right\} \enspace .
\]
Based on Bernstein's inequality, see Theorem \ref{thm:bernstein} in Appendix \ref{app:tech}, \eqref{eq:plugin_proof_2_app} hold if
\[
m \ge \max \left\{ \frac{8 \log \tfrac{2}{\delta}}{\epsilon^2}, \frac{4736\cdot \sqrt[4]{\log \tfrac{2}{\delta}}}{\sqrt{\epsilon}}, \frac{6528\sqrt[3]{F_{7/2} (\mathbf{p}) \log \tfrac{2}{\delta}} }{\epsilon^{4/3}} \right\}
\]
which concludes the proof.
To simplify the last terms, alternatively we can apply Hoeffding's inequality which yields that \eqref{eq:plugin_proof_2_app} holds whenever
\[
m \ge \frac{8}{\epsilon^2}\log \frac{2}{\delta} \enspace .
\]
Finally note that $32/\epsilon \le  8/\epsilon^2 \log 2/\delta$ for any $\epsilon, \delta \in (0,1]$ which concludes the proof.
\end{proof}

Note that as soon as $(F_3(X)- F_2(X)^2)/5 \le \epsilon $, the second term of the sample complexity of Theorem \ref{thm:ustat_estimatorsamp} becomes dominant, and thus the sample complexity in these parameter regime is $O(\ln(1/\delta)/\epsilon)$. In addition to this, it is easy to see that the first tern of the sample complexity is zero when $X$ is distributed uniformly.

\subsection{Testing by learning}
\label{subseq:testing-by-learning}
Testing by learning consists of estimating the parameter itself with a small additive error which allows us to distinguish between null $H_0$ and alternative hypothesis $H_1$. This approach had been found to be optimal in several testing problem~\cite{Busa-FeketeFSZ21}, as it is also optimal in this case based on the lower bound presented in the previous section. We considered several estimators for Collision entropy which can be used in a batch testing setup by setting the sample size so as the additive error of the estimator is smaller than $\epsilon/2$. In this way, we can distinguish between $H_0$ and $H_1$ as expected. The confidence interval of each estimator does depend on some frequency moment of the underlying distribution which can be upper worst case upper bounded. For example, the plug-n estimator sample complexity $m$ is $1600/\epsilon^2$ if $e^{-199} \le \delta$.

\section{Private sequential tester}
\label{app:private_seq}
We present the private sequential tester algorithm. We use the same hashing which is used in Mechanism \ref{alg:second}. The difference is that we do not create super group but the estimate for $C(\bp)$ is computed based on all users together. That is why the hashing in Line \ref{line:hashing} of Algorithm \ref{alg:psq} does depend only on the salt and observed sample. Since we use hashed data, the statistics is biased. Therefore we compute the biased null hypothesis $c_0$ in Line \ref{line:biased} and we also take into account that in the hash space the support of the test statistics scales with $2r$.

\begin{algorithm}[!h]
\caption{Private Sequential Tester (PSQ) \label{alg:psq}}
\begin{algorithmic}[1] 
\STATE {\bf Given:} Null hypothesis value $c_0$, confidence level $\delta \in [0, 1]$, privacy parameters $\alpha \ge 0, \beta \in [0, 1]$.
\STATE Set $c= \frac{c_0}{2r} + 1/2 $ with $r = 6 \left(\frac{e^\alpha + 1}{e^\alpha - 1}\right)^2\log \frac{4}{\beta}$ \label{line:biased}
\FOR{$i=1, 2, 3, \dots$}
\STATE User $i$ chooses salt $s_i$ uniformly at random from $\{1, \ldots, r\}$.
\STATE Draw sample $x_i$ from distribution $\bp$.
\STATE User $i$ sends hash value $v_i = h(\langle s_i, x_i \rangle)$. \label{line:hashing}
\STATE Let $T_i = \sum_{j=1}^{i-1} \indic{v_i = x_j} - 2(i - 1)c$.
\IF{$\left\lvert \frac{2}{i(i-1)}\sum_{j=1}^i T_j \right\rvert > 3.2 \sqrt{2r \cdot \frac{\log \log i + 0.72 \log (20.8 /\delta)}{i}}$} \label{alg1:reject_psq}
\STATE Reject the null hypothesis. 
\ENDIF
\ENDFOR
\end{algorithmic}
\end{algorithm}

\begin{proof}
We use the same privatization which is applied in Mechanism \ref{alg:second}, but super groups are not used here because we do not use the technique of Median-of-Means in Algorithm \ref{alg:psq}. 

First, it be shown that 
\begin{align*}
F_2(X) &= 2r \cdot \left(F_2(V) - \frac{1}{2}\right)
\end{align*}
where $X$ is the original random variable and $V$ is the privatized by using hashing.
This implies that if the expected value of $F(X_0) = c_0$, then the test statistic of Algorithm \ref{alg:psq} is equal to $c_0/2r +1/2$ which is computed in Line \ref{line:biased}. Accordingly, our algorithm is testing whether the U-statistics is close to this biased value $c_0/2r +1/2$ or not.

Therefore as long as we construct a confidence time-uniform interval for $F_2 (V)$, this readily implies a confidence interval for $F_2(X)$. However, we need to take into account that the support of $F_2 ( V)$ is $[0, 2r]$ which is $O(\log \frac1\beta)$. This and Theorem \ref{thm:general_closeness_thm_simple} and Theorem \ref{thm:second_error} implies the sample complexity bound. The privacy guaranty implied by Theorem \ref{thm:second_privacy}.
\end{proof}

\section{Technical tools}
\label{app:tech}

\subsection{LeCam's lower bound}
\label{app:lecam}

 Let $\hat{\theta}_n = \hat{\theta} ( x_1, \dots, x_n )$ such that $\hat{\theta}_n : (\Sigma^d)^n \mapsto \mathbb{R}$ be an estimator using $n$ samples.

\begin{thm}{[Le Cam's theorem]}
Let $\mathcal{P}$ be a set of distributions. Then, for any pair of distributions $P_0, P_1 \in \mathcal{P}$, we have
\[
\inf_{\hat{\theta}} \max_{P\in \mathcal{P}} \mathbb{E}_P \left[ d( \hat{\theta}_n (P), \theta(P) )  \right] \ge \frac{d( \theta(P_0), \theta(P_1) )}{8} e^{-n d_{\text{KL}} (P_0, P_1)},
\]
where $\theta ( P)$ is a parameter taking values in a metric space with metric $d$, and $\hat{\theta}_n$ is the estimator of $\theta $ based on $n$ samples.
\end{thm}

\subsection{Bersntein's bound}

The following form of Bernstein's bound can be derived from Theorem 1.4 of \cite{DeAl09}.
\begin{thm}{(Bernstein's bound)} \label{thm:bernstein}
Let $X_1, \dots, X_n$ be i.i.d. random variables, and $\forall i\in [n], \vert X_i - \mathbb{E} [ X_i ] \vert \le b$ and $\mathbb{E}[ X_i ] = \mu $. Let $\sigma^2 = \mathbb{V}[X_i]$. Then with probability at least $1-\delta$ it holds that
\[
\left\vert \frac{1}{n} \sum_{i=1}^n X_i - \mu \right\vert \le \sqrt{\frac{4\sigma^2 \ln \tfrac{2}{\delta}}{n}} + \frac{4b \ln \tfrac{2}{\delta}}{3n} \enspace .
\]
\end{thm}

\end{document}